%% file: Mar1-arxiv.tex
\documentclass[11pt]{jmlr}
\usepackage{algorithm,algorithmic}
\usepackage[small]{caption}

\newfloat{algorithm}{t}{lop}

\newcommand{\tilA}{A}
\newcommand{\trA}{A^*}
\newcommand{\trx}{x^*}
\newcommand{\tilx}{x}
\newcommand{\wtilx}{x}
\newcommand{\tilX}{X}
\newcommand{\trxs}[1]{{x^{*(#1)}}}
\newcommand{\trX}{X^*}
\newcommand{\xlbd}{C}
\newcommand{\calF}{\mathcal{F}}

\newcommand{\sgn}{\mbox{sgn}}

\newcommand{\near}{{near} }

\newcommand\E{\mathbf{E}}
\newcommand{\Exp}{\mathop{\mathbf E}\displaylimits}
\newcommand\inner[1]{\langle #1 \rangle}

\newcommand\supp{\mbox{supp}}
\newcommand\poly{\mbox{poly}}

\newtheorem{claim}[theorem]{Claim}
\usepackage{hyperref}
\newcommand{\ignore}[1]{}

\renewcommand{\Pr}{\mathop{\bf Pr\/}}                    % should we change these to \mathbb for consistency of single-letter functionals
\renewcommand{\E}{\mathop{\bf E\/}}

\renewcommand{\poly}{\mathrm{poly}}

\def\shownotes{1}  %set 1 to show author notes
\ifnum\shownotes=1
\newcommand{\authnote}[2]{{ $\ll$\textsf{\footnotesize #1 notes: #2}$\gg$}}
\else
\newcommand{\authnote}[2]{}
\fi
\newcommand{\Tnote}[1]{{\authnote{Tengyu}{#1}}}

\makeatletter
\floatstyle{ruled}
\newfloat{fragment}{H}{lop}
\floatname{fragment}{Algorithm}
\renewcommand{\floatc@ruled}[2]{\vspace{2pt}{\@fs@cfont \#1.\:} \#2 \par
 \vspace{1pt}}
\makeatother

\renewenvironment{proof}{\medskip\noindent{\textbf{Proof:}}} {$\blacksquare$\vskip \belowdisplayskip}

\title{Simple,  Efficient, and Neural Algorithms for Sparse Coding}
 \author{Sanjeev Arora \thanks{arora@cs.princeton.edu, Princeton University, Computer Science Department and Center for Computational Intractability. This work was supported in part by NSF grants CCF-0832797, CCF-1117309,
CCF-1302518, DMS-1317308,  Simons Investigator Award, and Simons Collaboration Grant.}
\and
Rong Ge \thanks{rongge@microsoft.com, Microsoft Research, New England. }
\and
Tengyu Ma \thanks{tengyu@cs.princeton.edu, Princeton University, Computer Science Department and Center for Computational Intractability. This work was supported in part by NSF grants CCF-0832797, CCF-1117309,
	CCF-1302518, DMS-1317308,  Simons Investigator Award, and Simons Collaboration Grant.}
\and
Ankur Moitra \thanks{moitra@mit.edu, Massachusetts Institute of Technology, Department of Mathematics and CSAIL. This work was supported in part by a grant from the MIT NEC Corporation and a Google Research Award.}
 }

\date{\today}

\begin{document}

\maketitle

		%\vspace{-0.15in}
\begin{abstract}
{\em Sparse coding} is a basic task in many fields including signal processing, neuroscience and machine learning where the goal is to learn a basis that enables a sparse representation of a given set of data, if one exists. Its standard formulation is as a non-convex optimization problem which is solved in practice by heuristics based on alternating minimization.  Recent work has
resulted in several algorithms for sparse coding with provable guarantees, but somewhat surprisingly these are outperformed 
by the simple alternating minimization heuristics. 
Here we give a general framework for understanding alternating minimization which we leverage to analyze existing heuristics and to design new ones also with provable guarantees. Some of these algorithms seem implementable on simple
neural architectures, which was the 
%We study this problem in a natural generative model, and are able e give the first {\em neurally plausible} algorithm \---- %closely related to the 
original motivation of~\cite{OF} in introducing sparse coding. 
%\---- that (provably) converges to a globally optimal sparse code. 
We also give the first efficient algorithm for sparse coding that works almost up to the information theoretic limit for sparse recovery on incoherent dictionaries. All previous algorithms that approached or surpassed this limit run in time exponential in some natural parameter. Finally, our algorithms improve upon the sample complexity of existing approaches. 
We believe that our analysis framework will have applications in other settings where simple iterative algorithms are used.
\end{abstract}

\thispagestyle{empty}

%\newpage

%\setcounter{page}{1}

\input{introstoc}

%\input{introflat}
%\input{modelflat}
%\input{mainalg}
\input{mainalgflat}

\input{indicatorincludeflat10}

\input{newinitinclude10}

\input{togetherincludeflat10}

\input{init10}

%\newpage
\vspace{-0.1in}
\section*{Conclusions}
Going beyond $\sqrt{n}$ sparsity requires new ideas as alternating minimization appears to break down.  Mysterious properties of alternating minimization are also left to explore, such as why  a random initialization works. %Also, can we prove convergence bounds without needing fresh samples in each iteration? 
Are these heuristics information theoretically optimal in terms of their sample complexity?  Finally, can we analyse energy minimization in other contexts as well?
% There are many other applications of it that pertain to learning a ``nice" representation of a given set of data, where no provable guarantees are known. 

%Several mysteries about the descent heuristics also defy explanation. Why do they work starting with
%trivial initializations (e.g., using a random set of samples as the initial code matrix)? Can we prove convergence
%without needing fresh samples at each iteration (this is trivially possible in principle if we increase the initial sample size, but that defeats the purpose). Are these heuristics information theoretically optimal
%in sample complexity? 

%Finally, energy minimization heuristics (including alternating minimization) are ubiquitous in variety of disciplines, and it would be nice to analyse some others. In particular, sparse coding has the form
%of \textquotedblleft finding a better basis for the data\textquotedblright and many other problems of this
%form would be a good next target. 

%\section*{Acknowledgements} 
%We are grateful to  Dmitri Chklovskii and Sebastian Seung for useful discussions about neural computation. 

\bibliography{dictionary}

\newpage

\appendix

%\input{convexityflat10} 

\input{supportflat}

%\input{togetherincludeflat10}

\input{newinitincludeapp}

\input{initapp10}

\input{samplecomplex10}

\input{neuralarchitecture}

\end{document}

%% file: introstoc.tex
\section{Introduction}

%\subsection{Background}

{\em Sparse coding} or {\em dictionary learning} consists of learning to express (i.e., {\em code}) a set of input vectors, say image patches, as linear combinations of a {\em small} number  of vectors chosen from a large {\em dictionary}. It is a basic task in many fields.
In signal processing, a wide variety of signals turn out to be sparse in an appropriately chosen basis (see references in \cite{M}). In neuroscience, sparse representations are 
believed to improve energy efficiency of the brain by allowing most neurons to be inactive at any given time. 
 In machine learning, imposing sparsity as a constraint on the representation is a useful way to avoid {\em over-fitting}. Additionally, methods for sparse coding can be thought of as a tool for {\em feature extraction} and are the basis for a number of important tasks in image processing such as segmentation, retrieval, de-noising and super-resolution (see references in \cite{E}), as well as a  building block for some deep learning architectures \cite{RBL}. It is also a basic problem in linear algebra itself since it involves finding a better basis.
 %Finally, one of the original motivations for this problem comes from neuroscience where it is widely believed that sparse representations %play an essential feature of neural coding of information:  information is carried by a small subset of active neurons, which is energy %efficient. 

%One of the central tasks that arises throughout machine learning, signal processing and statistics is to learn a sparse representation for a g %iven collection of data. This problem has many names and is often called {\em dictionary learning} or {\em sparse coding}, and was first introduced in the 
The notion was introduced by neuroscientists 
%Olshausen and Field 
\cite{OF} who formalized it as follows: Given a dataset $y^{(1)}, y^{(2)}, \ldots, y^{(p)} \in \Re^n$, 
our goal is to find a  set of basis vectors $A_1, A_2, \ldots, A_m \in \Re^n$
and sparse coefficient vectors $x^{(1)}, x^{(2)}, \ldots, x^{(p)} \in \Re^m$ that minimize
the {\em reconstruction error} 
	%	\vspace{-0.05in}
\begin{eqnarray}
\sum_{i=1}^p \|y^{(i)} - A \cdot x^{(i)}\|^2_2 + \sum_{i=1}^p S(x^{(i)}) \label{eqn:OFenergy}
\end{eqnarray}

where  $A$ is the $n\times m$ {\em coding matrix} whose $j$th column is $A_j$ and  $S(\cdot)$ is a nonlinear penalty function that is used to encourage sparsity. This function is nonconvex because both $A$ and the $x^{(i)}$'s
are unknown. Their paper as well as subsequent work chooses $m$ to be larger than $n$ (so-called {\em overcomplete} case) because this allows greater flexibility in adapting the representation to the data.  We remark that sparse coding should not be confused with the related \---- and usually easier \---- problem of finding the sparse representations of the
$y^{(i)}$'s {\em given} the coding
matrix $A$,  variously called {\em compressed sensing} or {\em sparse recovery} \cite{CRT, CT}.

Olshausen and Field also gave a local search/gradient descent heuristic for trying to minimize the nonconvex energy function~(\ref{eqn:OFenergy}). %, which is nonconvex. 
They gave experimental evidence that it produces coding matrices for image patches that resemble known features (such as Gabor filters) in $V1$ portion of the visual
cortex. A related paper of the same authors%Olshausen and Field 
~\cite{OF2} (and also%Lewicki and Sejnowski 
~\cite{LS}) places sparse coding in a more familiar
{\em generative model} setting whereby the data points $y^{(i)}$'s are assumed to be probabilistically generated according to a model
$y^{(i)} = \trA \cdot \trxs{i} + \mbox{noise}$ where $\trxs{1}, \trxs{2}, \ldots,\trxs{p}$ are samples from some appropriate distribution and $\trA$ is an unknown code. Then one can define the maximum likelihood estimate, 
and this leads to a different and usually more complicated energy function \---- and associated heuristics \---- compared to (\ref{eqn:OFenergy}).

Surprisingly, maximum likelihood-based approaches seem unnecessary in practice and local search/gradient descent on the energy function~(\ref{eqn:OFenergy}) with hard constraints works well,
as do related algorithms such as MOD~\cite{AEB} and  $k$-SVD~\cite{EAH}.  In fact these methods are so effective
that sparse coding is considered in practice to be a solved problem, even though it has no 
polynomial time algorithm per se.
%it had no polynomial-time
%algorithm.
%was no algorithm known that is
%provably polynomial time. 
%has been an important
%open problem for theory.
%Thus it has been an important problem for theory
%Of course, this is the case for nonconvex problems in many domains:
%heuristic algorithms work well, and this success is mysterious for theory.
%theoretically explaining their performance is  open.
%They work well, but to understand why is still an open question in most settings.

\paragraph{Efficient Algorithms vs Neural Algorithms.}
Recently, there has been rapid progress on designing  polynomial time algorithms for sparse coding with provable guarantees (the
relevant papers are discussed below).
All of these adopt the generative model viewpoint sketched above. %, and run in polynomial time. 
But the surprising success of the simple descent heuristics has remained largely unexplained.
Empirically, these heuristics far out perform \---- in running time, sample complexity, and solution quality \---- the new algorithms,
and this (startling) observation was in fact 
the starting point for the current work.

Of course, the famous example of {\em simplex} vs {\em ellipsoid} for linear programming reminds us that it can be much more challenging to 
analyze the behavior of an empirically successful algorithm than it is to
to design a new polynomial time algorithm from scratch! 
But for sparse coding the simple intuitive heuristics are important for another reason beyond just their algorithmic efficiency: they appear to be implementable in neural architectures. (Roughly speaking, this means that the algorithm stores the code matrix $A$ as synapse weights
in a neural network and updates the entries using  differences in potentials of the synapse's endpoints.) Since neural computation \---- and also deep learning \---- have proven to be difficult to analyze in general,
analyzing sparse coding thoroughly seems to be a natural first step for theory.  Our algorithm is a close relative of the
Olshausen-Field algorithm and thus inherits its neural implementability; see Appendix~\ref{sec:neural} for further discussion. %where we sketch how our algorithm is {\em neurally plausible}. 
%Neural plausibility is an important reason to analyse the convergence properties of the energy minimization %heuristics, and

Here we present a rigorous analysis of the simple energy minimization heuristic, 
% . This raises hope that
%i t may be possible to give rigorous analysis of neural computation, instead of solely relying upon
%empirical evidence. As 
and as a side benefit this yields bounds on running time and sample complexity for sparse coding that are 
better (in some cases, dramatically so) than the algorithms in recent papers. 
This adds to the recent literature on analyzing alternating minimization \cite{JNS, H, NJS, NNSAJ} but these work in a setting where there is a convex program that is known to work too, and in our setting, the only known convex program runs in time exponential in a natural parameter \cite{BKS}.

\subsection{Recent Work}

A common thread in recent work on sparse coding is to assume a generative model; the precise details vary, but each has the property that given enough samples the solution is essentially unique. %Spielman, Wang and Wright 
\cite{SWW} gave an algorithm that succeeds when $\trA$ has full column rank (in particular $m \leq n$) which works up to sparsity roughly $\sqrt{n}$. However this algorithm is not applicable in the more prevalent overcomplete setting. 
%Arora, Ge and Moitra 
\cite{AGM} and %Agarwal et al.
~\cite{AAN, AAJNT} independently gave algorithms in the overcomplete case assuming that $\trA$ is $\mu$-incoherent (which we define in the next section). The former gave an algorithm that works up to sparsity $n^{1/2 - \gamma}/\mu$ for any $\gamma > 0$ but the running time is $n^{\Theta(1/\gamma)}$;  
%Agarwal et al.~
\cite{AAN, AAJNT} gave an algorithm that works up to sparsity either $n^{1/4}/\mu$ or $n^{1/6}/\mu$ depending on the particular assumptions on the model. These works also analyze alternating minimization but assume that it starts from an estimate $\tilA$ that is column-wise $1/\mbox{poly}(n)$-close to $\trA$, in which case the objective function is essentially convex. %There is a natural barrier at sparsity $n^{1/2}/2\mu$ since this is the threshold at which $\trx$ can be uniquely recovered given $\trA$ and $y = \trA \trx$ \cite{DH, DE, GN}. 
 
%Barak, Kelner and Steurer 
\cite{BKS} gave a new approach based on the sum-of-squares hierarchy that works for  sparsity up to $n^{1-\gamma}$ for any $\gamma > 0$. But in order to output an estimate that is column-wise $\epsilon$-close to $\trA$ the running time of the algorithm is $n^{1/\epsilon^{O(1)}}$. In most applications, one needs to set (say) $\epsilon = 1/k$ in order to get a useful estimate. However in this case their algorithm runs in exponential time. The sample complexity of the above algorithms is also rather large, and is at least $\Omega(m^2)$ if not much larger. Here we will give simple and more efficient algorithms based on alternating minimization whose column-wise error decreases geometrically, and that work for sparsity up to $n^{1/2}/\mu \log n$. We remark that even empirically, alternating minimization does not appear to work much beyond this bound. 

%\vspace{-0.1in}
\subsection{Model, Notation and Results}

%\paragraph{Our Model}
We will work with the following family of generative models (similar to those in earlier papers)\footnote{The casual reader should
just think of $\trx$ as being drawn from some distribution that has independent coordinates. Even in this simpler setting ---which has
polynomial time algorithms using {\em Independent Component Analysis}---we do not know of any rigorous analysis of heuristics like
Olshausen-Field. The earlier papers were only interested in polynomial-time algorithms, so did not wish to assume independence.}:

%\begin{model}
\paragraph{Our Model}
Each sample is generated as $y = \trA \trx + \mbox{noise}$ where $\trA$ is a ground truth dictionary and $\trx$ is drawn from an unknown distribution $\mathcal{D}$ where
\begin{itemize}

\item[(1)] the support $S = \mbox{supp}(\trx)$ is of size at most $k$, $\Pr[i\in S] = \Theta(k/m)$ and $\Pr[i,j\in S] = \Theta(k^2/m^2)$ 

\item[(2)] the distribution is normalized so that $\E[\trx_i|\trx_j \ne 0] = 0$; $\E[{\trx_i}^2|\trx_i\ne 0] = 1$ and when $\trx_i \ne 0$, $|\trx_i|\ge C$ for some constant $C \le 1$ and

\item[(3)] the non-zero entries are pairwise independent and subgaussian, conditioned on the support.

\item[(4)] The noise is Gaussian and independent across coordinates.
\end{itemize}
%\end{model}

\noindent Such models are natural since the original motivation behind sparse coding was to discover a code whose representations have the property that the coordinates are almost independent. We can relax most of the requirements above, at the expense of further restricting the sparsity, but will not detail such tradeoffs.

The rest of the paper ignores the iid noise: it has little effect on our basic steps like computing
inner products of samples or taking singular vectors, and easily tolerated so long as it stays smaller
than the \textquotedblleft signal.\textquotedblright 
%We assume that the noise is Gaussian and independent across coordinates, and if its magnitude is much smaller %than the signal then it does not affect our main results. 

We assume $\trA$ is an incoherent dictionary, since these are widespread in signal processing \cite{E} and statistics \cite{DH}, and include various families of wavelets, Gabor filters as well as randomly generated dictionaries. 

\begin{definition}
An $n\times m$ matrix $A$ whose columns are unit vectors is $\mu$-incoherent if for all $i \ne j$ we have $\langle A_i, A_j \rangle \leq \mu/\sqrt{n}.$
\end{definition}

\noindent We also require that $\|\trA\| = O(\sqrt{m/n})$. However this can be relaxed within polylogarithmic factors by tightening the bound on the sparsity by the same factor.  Throughout this paper we will say that $\tilA^s$ is $(\delta, \kappa)$-near to $\trA$ if after a permutation and sign flips its columns are within distance $\delta$ and we have $\|\tilA^s - \trA\| \leq \kappa \|\trA\|$. See also Definition~\ref{def:near}. We will use this notion to measure the progress of our algorithms. Moreover we will use $g(n) = O^*(f(n))$ to signify that $g(n)$ is upper bounded by $C f(n)$ for some small enough constant $C$. Finally, throughout this paper we will assume that $k \leq O^*(\sqrt{n}/\mu \log n)$ and $m = O(n)$. 
Again, $m$ can be allowed to be higher by lowering the sparsity. {\em We assume all these conditions in  our main theorems. }

\paragraph{Main Theorems} In Section~\ref{sec:schema} we give a general framework for analyzing alternating minimization. 
 %There has been considerable recent progress on analyzing alternating minimization for {\em matrix completion} \cite{JNS, H}, {\em phase retrieval} \cite{NJS} and {\em sparse principal component analysis} \cite{NNSAJ}. However these operate in a setting where there was already a convex program that was known to work too, and in our setting the only known convex program runs in time exponential in $1/\epsilon$ \cite{BKS} (recall that $\epsilon$ is the accuracy).
Instead of thinking of the algorithm as trying to minimize a known non-convex function, we view it as trying to minimize an {\em unknown} convex function. 
%The challenge is that we know neither the function nor its gradient, but nevertheless we prove that 
Various update rules are shown to provide good approximations to the gradient of the unknown function.
% in our generative model. 
See Lemma~\ref{lem:OF-simplified-update}, Lemma~\ref{lem:OFupdate} and Lemma~\ref{lem:OFupdate-noerror} for examples. We then leverage our framework to analyze existing heuristics and to design new ones also with provable guarantees.  In Section~\ref{sec:update_rule}, we prove:

\begin{theorem} \label{thm:mainneural}
There is a {\em neurally plausible} algorithm which when initialized with an estimate $\tilA^0$ that is $(\delta, 2)$-near to $\trA$ for $\delta = O^*(1/\log n)$, converges at a geometric rate to $\trA$ until the column-wise error is $O(\sqrt{k/n})$. Furthermore the running time is $O(mnp)$ and the sample complexity is $p = \widetilde{O}(mk)$ for each step. 
\end{theorem}

\noindent   Additionally we give a neural architecture implementing our algorithm in Appendix~\ref{sec:neural}.  To the best of our knowledge, this is the first neurally plausible algorithm for sparse coding with provable convergence.
% and the fact that it works sheds new light on how sparse coding might be accomplished in nature.

Having set up our general framework and analysis technique we can use it on other variants of alternating minimization.
%, and to follow the same steps in the analysis. In
 Section~\ref{sec:remove} gives  a new update rule whose bias (i.e., error) is negligible:

\begin{theorem}
There is an algorithm which when initialized with an estimate $\tilA^0$ that is $(\delta, 2)$-near to $\trA$ for $\delta = O^*(1/\log n)$, converges at a geometric rate to $\trA$ until the column-wise error is $O(n^{-\omega(1)})$. Furthermore each step runs in time $O(mnp)$ and the sample complexity $p$ is polynomial. 
\end{theorem}

\noindent This algorithm is based on a modification where we carefully project out components along the column currently being updated.  We complement the above theorems by revisiting the Olshausen-Field rule and analyzing a variant of it in Section~\ref{sec:ofanalysis} (Theorem~\ref{thm:OFupdate}). However its analysis is more complex because we need to bound some quadratic error terms. It uses convex programming.
 % that have to be bounded. % are quadratic. 

What remains is to give a method to initialize these iterative algorithms. We give a new approach based on pair-wise reweighting and we prove that it returns an estimate $\tilA^0$ that is $(\delta, 2)$-near to $\trA$ for $\delta = O^*(1/\log n)$ with high probability. As an additional benefit, this algorithm can be used even in settings where $m$ is not known and this could help solve another problem in practice \---- that of {\em model selection}. In Section~\ref{sec:init} we prove:

\begin{theorem}
There is an algorithm which returns an estimate $\tilA^0$ that is $(\delta, 2)$-near to $\trA$ for $\delta = O^*(1/\log n)$. Furthermore the running time is $\widetilde{O}(mn^2p)$ and the sample complexity $p = \widetilde{O}(mk)$.
\end{theorem}

\noindent This algorithm also admits a neural implementation, which is sketched in Appendix~\ref{sec:neural}.
The proof currently requires a projection step that increases the run time though we suspect it is not needed. 

We remark that these algorithms work up to sparsity $O^*(\sqrt{n}/\mu \log n)$ which is within a logarithmic factor of the information theoretic threshold for sparse recovery on incoherent dictionaries \cite{DH, GN}. All previous known algorithms that approach \cite{AGM} or surpass this sparsity \cite{BKS} run in time exponential in some natural parameter. Moreover, our algorithms are simple to describe and implement, and involve only basic operations. We believe that our framework will have applications beyond sparse coding, and could be used to show that simple, iterative algorithms can be powerful in other contexts as well by suggesting new ways to analyze them.

%% file: mainalgflat.tex
%\section{Alternating Minimization}
%\label{sec:OFframework}
%Here we describe alternating minimization, which was the approach taken by
\section{Our Framework, and an Overview}\label{sec:schema}

%\subsection{The Recipe}

Here we describe our framework for analyzing alternating minimization. The generic scheme we will be interested in is given in Algorithm~\ref{eqn:generic} and it alternates between updating the estimates $\tilA$ and $\tilX$. It is a heuristic for minimizing the non-convex function in (\ref{eqn:OFenergy}) where the penalty function is a hard constraint. The crucial step is if we fix $\tilX$ and compute the gradient of (\ref{eqn:OFenergy}) with respect to $\tilA$, we get:
$$\nabla_{\tilA}\mathcal{E}(\tilA,  \tilX)=\sum_{i=1}^p -2(y^{(i)}-\tilA \tilx^{(i)})(\tilx^{(i)})^T.
$$
\setlength{\textfloatsep}{5pt}
\begin{algorithm}\caption{Generic Alternating Minimization Approach}\label{eqn:generic}
	\vspace{0.1in}
	\textbf{Initialize} $\tilA^0$ 
	
	\textbf{Repeat} for $s = 0, 1, ..., T$
	%\vspace{-0.1in}
	\begin{align*}
		\textbf{\textbf{Decode: }} & \mbox{Find a sparse solution to } \tilA^s  \tilx^{(i)}  = y^{(i)} \mbox{ for $i = 1, 2, ..., p$}\nonumber\\
		&\mbox{Set } \tilX^{s} \mbox{ such that its columns are } \tilx^{(i)} \mbox{ for $i = 1, 2, ..., p$}\nonumber\\
		\textbf{\textbf{Update: }} & \tilA^{s+1} = \tilA^s - \eta g^{s} \mbox{ where } g^{s} \mbox{ is the gradient of } \mathcal{E}(\tilA^s,  \tilX^s) \mbox{ with respect to } \tilA^s \nonumber
	\end{align*}
		%\vspace{-0.2in}
\end{algorithm}

We then take a step in the opposite direction to update $\tilA$. 
 Here and throughout the paper $\eta$ is the learning rate, and needs to be set appropriately. 
The challenge in analyzing this general algorithm is to identify a suitable \textquotedblleft measure of progress\textquotedblright \---- called a Lyapunov function in dynamical systems and control theory \---- and show that
it improves at each step (with high probability over the samples). % ({\em A priori} it is conceivable that most iterations wander around 
%or even go a bit in the wrong direction, and occasionally make large steps in the correct direction. 
%If so, theoretical analysis would be difficult indeed.) 
%{\em occasionally} reduce the measure of progress, which would make theoretical analysis very difficult. 
We will measure the progress of our algorithms by the maximum column-wise difference between $\tilA$ and $\trA$. 

In the next subsection, we identify sufficient conditions that guarantee progress. They are inspired by proofs in convex optimization. %, that are sufficient to ensure that this decreases at each step. Our approach is to view
We view Algorithm~\ref{eqn:generic} as trying to minimize an {\em unknown} convex function,
specifically $f(\tilA) =\mathcal{E}(\tilA,  \trX)$, which is strictly convex and hence has a unique optimum that can be reached via gradient descent.  This function is unknown since the algorithm does not know $\trX$. The analysis will show that
the direction of movement is correlated with $\trA -A^s$, which in turn is the gradient of the above
function. 
%\Tnote{Actually we don't need to compare the direction of movement with the gradientof the above function for analysis. In analysis, we don'e %even need to define $f(A)$. }
An independent paper of~\cite{BWY14} proposes a  similar framework for  analysing EM algorithms for hidden variable models. 
The difference is that their condition is really about the geometry of the objective function, though ours is about the property of the direction of movement. Therefore we have the flexibility to choose different decoding procedures. This flexibility allows us to have a closed form of $X^s$ and obtain a useful functional form of $g^s$.
\iffalse 
The difference is that their algorithms are actually computing a gradient (specifically of
the true posterior distribution of the variables), whereas 
the alternating minimization approach described here is only heuristically following some improvement direction.
(It is not even a true gradient of the energy function.)  
\Tnote{the sentence above is not quite true (we also computed the true gradient, we just use a different decoding). How about the following: The difference is that their condition is really about the geometry of the objective function, though ours is about the property of the direction of movement. Therefore we have the flexibility to choose different decoding procedures. This flexiblity allows us to have a closed form of $X^s$ and obtain a useful functional form of $g^s$. }

That is why we need to work to show the correlation between this direction and the desired unknown gradient,
and need to allow bias. \Tnote{This is not quite true either.. We compare to $A^*-A^s$}. Thus our framework seems a bit more general (though still pretty straightforward).
Next we describe it.

%\iffalse 
\Tnote{I edited below a bit. Please take a look}

The direction of movement $g_s$ will turn out to be an {\em approximation} to the gradient of this unknown
function $f(A)$ in our generative model.
\fi
 The setup is reminiscent of {\em stochastic gradient descent}, which moves in a direction whose expectation is the gradient of a {\em known} convex function. By contrast, here the function $f()$ is unknown, and furthermore 
the expectation of $g^s$ is not the true gradient and has {\em bias}. 
%\Tnote{I added this sentence}Moreover, $f(A)$ is unknown but only defined for analysis purpose, as opposed to the setting of %stochastic gradient descent where the convex objective is known. 
Due to the bias, we will only be able to prove that our algorithms reach an approximate optimum up to some error 
whose magnitude is determined by the bias. We can make the bias negligible using more complicated
algorithms.
\subsection*{Approximate Gradient Descent}
Consider a general iterative algorithm that is trying to get to a desired solution  $z^*$
(in our case $z^* = \trA_i$ for some $i$). At step $s$ it starts with a guess $z^s$, 
computes some direction $g^s$, and updates its estimate as: $z^{s+1} = z^s - \eta g^s$. The natural progress measure is $\|z^* - z^s\|^2$, and below we will identify a sufficient condition for it to decrease in each step:

\begin{definition}\label{eqn:con}
A vector $g^s$ {\em is $(\alpha, \beta,\epsilon_s)$-correlated with} $z^*$ if
$$ \langle g^s, z^s-z^*\rangle \ge \alpha \|z^s-z^*\|^2 + \beta\|g^s\|^2 - \epsilon_s.$$
\end{definition}
\noindent{\em Remark:} The traditional analysis of convex optimization corresponds to the setting where $z^*$ is the global optimum of some convex function $f$, and $\epsilon_s = 0$. Specifically, if $f(\cdot)$ is $2\alpha$-strongly convex and $1/(2\beta)$-smooth, then $g^s = \nabla f(z^s)$ $(\alpha,\beta,0)$-correlated with $z^*$. Also we will refer to $\epsilon_s$ as the bias. 

% Under the above condition, we have the following theorem:
% In the previous section we gave update rules derived from $\nabla \widetilde{f}(B)$ that satisfy Condition~\ref{eqn:con}, and here we give a general analysis of their convergence. Our main theorem is:

\begin{theorem}\label{thm:tiger-convergence}
Suppose $g^s$ satisfies Definition~\ref{eqn:con} for $s = 1,2,\dots,T$, and $\eta$ satisfies  $0< \eta \le 2\beta$ and $\epsilon =  \max_{s=1}^{T} \epsilon_s$. 
Then for any $s = 1,\dots,T$, 
$$ \| z^{s+1}-z^*\|^2 \leq (1-2\alpha\eta) \|z^s-z^*\|^2 + 2\eta \epsilon_s$$
In particular, the update rule above converges to $z^*$ geometrically with systematic error $\epsilon/\alpha$ in the sense that 
%$$\|z_s-z^*\|^2 \le (1-2\alpha\eta)^s \|z_0-z^*\| + \epsilon/\alpha.$$
$$\|z^s-z^*\|^2 \le (1-2\alpha\eta)^s \|z^0-z^*\|^2 + \epsilon/\alpha.$$
Furthermore, if $\epsilon_s < \frac{\alpha}{2} \|z^s-z^*\|^2$ for $s = 1,\dots,T$, then 
$$\|z^s-z^*\|^2 \le (1-\alpha\eta)^s \|z^0-z^*\|^2.$$
%$$\|x^s-x^*\|^2 \le (1-2\alpha\beta)^s R^2$$
\end{theorem}
The proof closely follows existing proofs in convex optimization: %and we also give an analysis for approximate projected gradient descent in Corollary~\ref{cor:tiger-convexity}. 

\begin{proof}[Proof of Theorem~\ref{thm:tiger-convergence}]
	We expand the error as
	\[
	\begin{array}{rl}
	\|z^{s+1} - z^*\|^2  &=  \|z^s-z^*\|^2 -2 \eta {g^s}^T(z^s-z^*) + \eta^2 \|g^s\|^2\\
	&= \|z^s-z^*\|^2 -\eta \left(2{g^s}^T(z^s-z^*) -\eta \|g^s\|^2\right) \\
	&\leq \|z^s-z^*\|^2 -\eta \left(2\alpha \|z^s-z^*\|^2 + (2\beta -\eta)\|g^s\|^2-2\epsilon_s\right) \quad \mbox{(Definition~\ref{eqn:con} and  $\eta \le 2 \beta$)}  \\
	&\leq \|z^s-z^*\|^2 -\eta \left(2\alpha \|z^s-z^*\|^2 -2\epsilon_s\right) \\
	& \leq (1-2\alpha\eta) \|z^s-z^*\|^2 + 2\eta \epsilon_s
	\end{array}
	\]
	%\Tnote{The proof might need slight change if we agree with the new phrasing of the statement.}
	Then solving this recurrence we have
	$\|z^{s+1}-z^*\|^2  \leq (1-2\alpha\eta)^{s+1} R^2 + \frac{\epsilon}{\alpha}$
	where $R = \|z^0-z^* \|$. And furthermore if $\epsilon_s < \frac{\alpha}{2} \|z^s-z^*\|^2$ we have instead
	$$ \|z^{s+1} - z^*\|^2 \leq (1-2\alpha\eta) \|z^s-z\|^2 + \alpha \eta \|z^s-z\|^2 = (1-\alpha\eta) \|z^s-z\|^2$$
	and this yields the second part of the theorem too.
\end{proof}

In fact, we can extend the analysis above to obtain identical results for the case of constrained optimization. Suppose we are interested in optimizing a convex function $f(z)$ over a convex set $\mathcal{B}$. 
The standard approach is to take a step in the direction of the gradient (or $g^s$ in our case) and then project into $\mathcal{B}$ after each iteration,
namely, replace $z^{s+1}$ by $\mbox{Proj}_\mathcal{B}\, z^{s+1}$ which is the closest point in $\mathcal{B}$ to $z^{s+1}$ in Euclidean distance.
It is well-known that if $z^*\in \mathcal{B}$, then $\|\mbox{Proj}_\mathcal{B}\; z - z^*\| \le \|z - z^*\|$. Therefore we obtain the following as an immediate corollary to the above analysis:

\begin{corollary}\label{cor:tiger-convexity}
	Suppose $g^s$ satisfies Definition~\ref{eqn:con} for $s = 1,2, \dots,T$ and set  $0< \eta \le 2\beta$ and $\epsilon =  \max_{s=1}^T \epsilon_s$. Further suppose that
	$z^*$ lies in a  convex set $\mathcal{B}$. Then the update rule $z^{s+1} = \mbox{Proj}_\mathcal{B}(z^s - \eta g^s)$ satisfies that for any $s = 1,\dots,T$, 
	$$\|z^s-z^*\|^2 \le (1-2\alpha\eta)^s \|z^0-z^*\|^2 + \epsilon/\alpha$$
	In particular, $z^s$ converges to $z^*$ geometrically with systematic error $\epsilon/\alpha$. 
	%$$\|x^s-x^*\|^2 \le (1-4\alpha\beta)^s \|z^0-z^*\| + \epsilon/(\alpha)$$
	Additionally if $\epsilon_s < \frac{\alpha}{2} \|z^s-z^*\|^2$ for $s = 1,\dots,T$, then 
	$$\|z^s-z^*\|^2 \le (1-\alpha\eta)^s \|z^0-z^*\|^2$$
	%$$\|x^s-x^*\|^2 \le (1-2\alpha\beta)^s R^2$$
\end{corollary}

What remains is to derive a functional form for various update rules and show that these rules move in a direction $g^s$ that approximately points in the direction of the desired solution $z^*$ (under the assumption that our data is generated from a stochastic model that meets certain conditions). 

%We defer the proof to Appendix~\ref{subsec:support}; it closely follows existing proofs in convex optimization, and we also give an analysis for approximate projected gradient descent in Corollary~\ref{cor:tiger-convexity}. 

\iffalse 
Of course, the algorithm does not know $\trX$, so it has access to neither the convex function itself nor its gradient. Our proof will regard the update direction 
$g_s$ as an {\em approximation} to the gradient of  this unknown convex  function. 
This approximation will  have a systemic error because the expectation of $g^s$ is not the true gradient. Nevertheless we give a general methodology to control the error, whereby it follows that the update rule converges because the corresponding update on the
unknown function converges. Our methodology allows us to plug in various update rules and get a wide range of new algorithmic results. 
\fi 

\subsection*{An Overview of Applying Our Framework}

%In alternating minimization, the usual approach is to take a step in the direction opposite to the gradient. \Tnote{The sentence above is not %quite true. In alt min, people take minimizers of both $A$ and $X$} 
Our framework clarifies that any improvement step meeting Definition~\ref{eqn:con} will also converge to an approximate optimum, which enables us to engineer other update rules that  turn out to be easier to analyze. Indeed we first analyze a simpler update rule with $g^{s} = \E[(y - \tilA^s \tilx) \sgn(\tilx^T)]$ in Section~\ref{sec:update_rule}. Here $\sgn(\cdot)$ is the coordinate-wise sign function. We then return to the Olshausen-Field update rule and analyze a variant of it in Section~\ref{sec:ofanalysis} using approximate projected gradient descent.  Finally, we design a new update rule in Section~\ref{sec:remove} where we carefully project out components along the column currently being updated. This has the effect of replacing one error term with another and results in an update rule with negligible bias. The main steps in showing that these update rules fit into our framework are given in Lemma~\ref{lem:OF-simplified-update}, Lemma~\ref{lem:OFupdate} and Lemma~\ref{lem:OFupdate-noerror}.

 How should the algorithm update $\tilX$? The usual approach is to solve a sparse recovery problem with respect to the current code matrix $\tilA$. However many of the standard basis pursuit algorithms (such as 
solving a linear program with an $\ell_1$ penalty) are difficult to analyze when there is error in the code itself. This is in part because the solution does not have a closed form 
  in terms of the code matrix.  Instead we take a much simpler approach to solving the sparse recovery problem 
which uses matrix-vector multiplication followed by thresholding: In particular, we set
$\tilx = \textrm{threshold}_{C/2}((\tilA^s)^Ty)$, where $\textrm{threshold}_{C/2}(\cdot)$ keeps only the coordinates whose magnitude is at least $C/2$ and zeros out the rest. Recall that the non-zero coordinates in $\trx$ have magnitude at least $C$.  
This decoding rule recovers the signs and support of $\tilx$ correctly provided that
$\tilA$ is column-wise $\delta$-close to $\trA$ for $\delta = O^*(1/\log n)$. See Lemma~\ref{lem:sign-correct}. 
% This works so long as $\tilA$ and $\trA$ are $\delta$-close, for .

 %This rule makes some sense is $\tilA$ is $\epsilon$-close to 
 
 The rest of the analysis can be described as follows: If the signs and support of $\tilx$ are recovered correctly, then alternating minimization makes progress in each step. In fact this holds each for much larger values of $k$ than we consider; as high as $n/(\log n)^{O(1)}$. (However, the explicit decoding rule fails 
 for $k> \sqrt{n}/\mu \log n$.) Thus it only remains to properly initialize $\tilA^0$ so that it is close enough to $\trA$ to let the above decoding rule succeed. In Section~\ref{sec:init} we give a new initialization procedure based on pair-wise reweighting that we prove works with high probability. This section may be of independent interest, since this algorithm can be used even in settings where $m$ is not known and could help solve another problem in practice \---- that of {\em model selection}. See Lemma~\ref{lem:initform}. 
 
 \iffalse

This brings us to the final issue of how to properly initialize $\tilA^0$, so that it is indeed
$O^*(1/\log n)$-close to $\trA$. Proper initialization is important in general for alternating minimization.
%since otherwise the above-mentioned decoding
%isn't guaranteed to work. 
In our case the primary issue is being able to carry out the analysis; the algorithm actually works (according to our experiments)  with the fairly 
trivial initialization of $\tilA^0$ using a random subset of samples $y^{(i)}$'s as the columns.
But proving convergence from this weak starting point is still  open.
In Section~\ref{} we describe a initialization using a modification of the trivial 
initialization: compute singular vectors of various matrices of weighted samples, followed by simple greedy clustering of
the resulting vectors. (NB: Neural algorithms for computing top singular vectors are known~\cite{}.)

\fi

%% file: indicatorincludeflat10.tex
\section{A Neurally Plausible Algorithm with Provable Guarantees}\label{sec:update_rule}
Here we will design and analyze a neurally plausible algorithm for sparse coding which is given in Algorithm~\ref{eqn:simplestupdate}, and we give a neural architecture implementing our algorithm in Appendix~\ref{sec:neural}.  The fact that such a simple algorithm provably works sheds new light on how sparse coding might be accomplished in nature. Here and throughout this paper we will work with the following measure of closeness:

\begin{definition}\label{def:near}
$\tilA$ is {\em $\delta$-close to} $\trA$ if there is a permutation $\pi: [m] \rightarrow [m]$ and a choice of signs $\sigma: [m] \rightarrow \{\pm 1\}$ such that
$ \|\sigma(i) \tilA_{\pi(i)} - \trA_i\| \leq \delta \mbox{ for all $i$}$
We say $\tilA$ is {\em $(\delta, \kappa)$-\near to} $\trA$ if in addition $\|\tilA-\trA\| \le \kappa \|\trA\|$ too. 
\end{definition}

\noindent This is a natural measure to use, since we can only hope to learn the columns of $\trA$ up to relabeling and sign-flips. In our analysis, we will assume throughout that $\pi(\cdot)$ is the identity permutation and $\sigma(\cdot) \equiv +1$ because our family of generative models is invariant under this relabeling and it will simplify our notation. 

Let $\sgn(\cdot)$ denote the coordinate-wise sign function and recall that $\eta$ is the learning rate, which we will soon set. Also we fix  both $\delta, \delta_0 = O^*(1/\log n)$. We will also assume that in each iteration, our algorithm is given a fresh set of $p$ samples. Our main theorem is:

\begin{theorem}\label{thm:simplestupdate}
Suppose that $\tilA^0$ is $(2\delta,2)$-\near to $\trA$ and that
$\eta = \Theta(m/k)$. Then if each update step in Algorithm~\ref{eqn:simplestupdate} uses $p = \widetilde{\Omega}(mk)$ fresh samples, we have
$$\Exp[\|\tilA_i^s - \trA_i\|^2]\le (1-\tau)^s \|\tilA^0_i - \trA_i\|^2 + O(k/n)$$
for some $0 < \tau < 1/2$ and for any $s=1,2, ... , T$. In particular it converges to $\trA$ geometrically, until the column-wise error is $O(\sqrt{k/n})$. 
\end{theorem}

Our strategy is to prove that $\widehat{g}^s$ is $(\alpha,\beta,\epsilon)$-correlated (see Definition~\ref{eqn:con}) with the desired solution $\trA$, and then to prove that $\| \tilA\|$ never gets too large. We will first prove that if $\tilA$ is somewhat close to $\trA$ then the estimate $\tilx$ for the representation almost always has the correct support. Here and elsewhere in the paper, we use ``very high probability"  to mean that an event happens with probability at least $1- 1/n^{\omega(1)}$. 

\begin{lemma}\label{lem:sign-correct}
Suppose that $\tilA^s$ is $\delta$-close to $\trA$.  Then with very high probability over the choice of the random sample $y = \trA\trx$:
$$\mbox{{\em sgn}}(\mbox{{\em threshold}}_{C/2}((\tilA^s)^Ty)  ) = \mbox{{\em sgn}}(\trx)$$ 
\end{lemma}

\noindent We prove a more general version of this lemma (Lemma~\ref{thm:supp}) in Appendix~\ref{subsec:support}; it is an ingredient in analyzing all of the update rules we consider in this paper. However this is just one step on the way towards proving that $\widehat{g}^s$ is correlated with the true solution. 

\setlength{\textfloatsep}{15pt}
\begin{algorithm}\caption{Neurally Plausible Update Rule}\label{eqn:simplestupdate}

\vspace{0.1in}
\textbf{Initialize} $\tilA^0$ that is $(\delta_0, 2)$-near to $\trA$

\textbf{Repeat}  for $s = 0, 1, ..., T$
%\vspace{-0.1in}
\begin{align*}
\textbf{\textbf{Decode: }} & \tilx^{(i)} = \textrm{threshold}_{C/2}((\tilA^s)^Ty^{(i)})  \textrm{ for $i = 1, 2, ... , p$}\nonumber\\
\textbf{\textbf{Update: }} & \tilA^{s+1} = \tilA^s - \eta \widehat{g}^{s} \mbox{ where } \widehat{g}^{s} = \frac{1}{p}\cdot\sum_{i=1}^p (y^{(i)} - \tilA^s \tilx^{(i)}) \sgn(\tilx^{(i)})^T \nonumber
\end{align*}
%\vspace{-0.1in}
\end{algorithm}

The next step in our proof is to use the properties of the generative model to derive a new formula for $\widehat{g}^s$ that is more amenable to analysis. %For simplicity, we start from the case when we are given infinite number of samples at each iteration. And 
We define $g^s$ to be the expectation of $\widehat{g}^s$
\begin{eqnarray}
g^s := \Exp[\widehat{g}^s] = \Exp[(y - A^s \wtilx) \sgn(\tilx)^T]
\end{eqnarray}
where $x := \textrm{threshold}_{C/2}((\tilA^s)^Ty)$ is the decoding of $y$. 
Let $q_i = \Pr[\trx_i\ne 0]$ and $q_{i,j} = \Pr[\trx_i\trx_j\ne 0]$, and define $p_i = \E[\trx_i\sgn(\trx_i)|\trx_i\ne 0]$. 

Here and in the rest of the paper,  we will let $\gamma$ denote any vector whose norm is negligible (i.e. smaller than $1/n^C$ for any large constant $C > 1$). This will simplify our calculations. Also let $\trA_{-i}$ denote the matrix obtained from deleting the $i$th column of $\trA$. The following lemma is the main step in our analysis. 
%\vspace{-0.05in}
\begin{lemma}\label{lem:OF-simplified-update}
Suppose that $\tilA^s$ is $(2\delta,2)$-\near to $\trA$. Then the update step in Algorithm~\ref{eqn:simplestupdate} takes the form 
$\Exp[\tilA_i^{s+1}] = \tilA_i^{s} - \eta g_i^s$
where $g_i^s = p_iq_i \left(\lambda^s_i\tilA_i^s - \trA_i +\epsilon^s_i \pm \gamma\right)$,  and $\lambda^s_i = \langle \tilA_i^s, \trA_i\rangle $ and
$$\epsilon_i^s = \Big( \tilA^s_{-i} \mbox{diag}(q_{i,j}) \left(\tilA_{-i}^s\right)^T \Big)\trA_i/q_i$$
Moreover the norm of $\epsilon^s_i$ can be bounded as $\|\epsilon^s_i\|\le O(k/n)$. 
\end{lemma}

\noindent 
Note that $p_iq_i$ is a scaling constant and $\lambda_i \approx 1$; hence from the above formula we should expect that $g_i^s$ is well-correlated with $\tilA_i^s - \trA_i$. 

\begin{proof}  
Since $\tilA^s$ is $(2\delta,2)$-\near to $\trA$, $\tilA^s$ is $2\delta$-close to $\trA$. We can now invoke Lemma~\ref{lem:sign-correct} and conclude that with high probability, $\sgn(\trx) = \sgn(\tilx)$.
Let $\calF_{\trx}$ be the event that $\sgn(\trx) = \sgn(\tilx)$, and let $\mathbf{1}_{\calF_{\trx}}$ be the indicator function of this event. 

To avoid the overwhelming number of appearances of the superscripts, let $B = \tilA^s$ throughout this proof. Then we can write $g_i^s =  \E[(y - B \wtilx) \sgn(\tilx_i)] $. Using the fact that $\mathbf{1}_{\calF_{\trx}} + \mathbf{1}_{\bar{\calF}_{\trx}} = 1$ and that $\calF_{\trx}$ happens with  very high probability:
\begin{eqnarray}
g_i^s &=&  \E[(y - B \wtilx) \sgn(\tilx_i) \mathbf{1}_{\calF_{\trx}}] +   \E[(y - B \wtilx) \sgn(\tilx_i)  \mathbf{1}_{\overline{\calF}_{\trx}}] \nonumber \\
&=& \E[(y - B \wtilx) \sgn(\tilx_i)  \mathbf{1}_{\calF_{\trx}}] \pm \gamma \label{eqn:g_s1}
\end{eqnarray}

The key is that this allows us to essentially replace $\sgn(\tilx)$ with $\sgn(\trx)$. Moreover,  let $S= \mbox{supp}(\trx)$. Note that when $\calF_{\trx}$ happens $S$ is also the support of $\tilx$. Recall that according to the decoding rule (where we have replaced $\tilA^s$ by $B$ for notational simplicity) $\tilx = \textrm{threshold}_{C/2}(B^T y)$. Therefore, $\tilx_S = (B^Ty)_S = B_S^T y = B_S^T\trA \trx$. Using the fact that the support of $\tilx$ is $S$ again, we have $B\tilx = B_S^TB_S\trA \trx$. Plugging it into equation (\ref{eqn:g_s1}):
%\vspace{-0.05in}
\begin{eqnarray*}
g_i^s &=& \E[(y - B \tilx) \sgn(\tilx_i) \mathbf{1}_{\calF_{\trx}}] \pm \gamma= \E[(I - B_S B_S^T ) \trA \trx\cdot  \sgn(\trx_i)  \mathbf{1}_{\calF_{\trx}}] \pm \gamma\\
&=& \E[(I - B_S B_S^T ) \trA \trx \cdot \sgn(\trx_i)  ]  - \E[(I - B_S B_S^T ) \trA \trx \cdot \sgn(\tilx_i) \mathbf{1}_{\bar{\calF}_{\trx}}] \pm \gamma\\
&=& \E[(I - B_S B_S^T ) \trA x\cdot \sgn(\trx_i) ]   \pm \gamma
\end{eqnarray*}
where again we have used the fact that $\calF_{\trx}$ happens with very high probability. Now we rewrite the expectation above using subconditioning where we first choose the support $S$ of $\trx$, and then we choose the nonzero values $\trx_S$. 
\begin{eqnarray*}
\E[(I - B_S B_S^T ) \trA \trx \cdot \sgn(\trx_i) ]  &=& \E_S \Big[ \E_{\trx_S}[(I - B_S B_S^T ) \trA \trx \cdot \sgn(\trx_i) | S ] \Big] \\
&=& \E [p_i (I - B_S B_S^T ) \trA_i]
\end{eqnarray*}
where we use the fact that $\E[\trx_i \cdot \sgn(\trx_i) | S] = p_i$. Let $R = S -\{i\}$. Using the fact that $B_SB_S^T  = B_iB_i^T + B_RB_R^T$, we can split the quantity above into two parts, 
\begin{eqnarray*}
g_i^s &=& p_i \E[(I - B_i B_i^T) \trA_i + p_i \E[B_R B_R^T ] \trA_i \\
&=& p_iq_i \Big(I - B_i B_i^T\Big)\trA_i +  p_i \Big( B_{-i} \mbox{diag}(q_{i,j}) B_{-i}^T \Big)\trA_i \pm \gamma.
\end{eqnarray*}

\noindent where $\mbox{diag}(q_{i,j})$ is a $m\times m$ diagonal matrix whose $(j,j)$-th entry is equal to $q_{i,j}$, and $B_{-i}$ is the matrix obtained by zeroing out the $i$th column of $B$. Here we used the fact that $\Pr[i\in S] = q_i$ and $\Pr[i,j\in S] = q_{ij}$. 

Now we set $B = \tilA^s$, and rearranging the terms, we have 
$g_i^s = p_iq_i \left(\langle \tilA_i^s, \trA_i\rangle\tilA_i^s - \trA_i +\epsilon^s_i \pm \gamma\right)$
where $\epsilon_i^s = \Big( \tilA^s_{-i} \mbox{diag}(q_{i,j}) \left(\tilA_{-i}^s\right)^T \Big)\trA_i/q_i$, which can be bounded as follows
$$\|\epsilon_i^s\|\le \|\tilA^s_{-i}\|^2 \max_{j\neq i} q_{i,j}/q_i \le O(k/m) \|\tilA^s\|^2 = O(k/n)$$
where the last step used the fact that  $\frac{\max_{i\neq j}q_{i,j}}{\min q_i } \le O(k/m)$, which is an assumption of our generative model.  
\end{proof}

\noindent We will be able to reuse much of this analysis in Lemma~\ref{lem:OFupdate} and Lemma~\ref{lem:OFupdate-noerror} because we have derived to for a general decoding matrix $B$.

In Section~\ref{sec:hypothesis} we complete the analysis of Algorithm~\ref{eqn:simplestupdate} in the infinite sample setting. In particular, in Section~\ref{sec:progress}, we prove that if $\tilA^s$ is $(2\delta,2)$-\near to $\trA$ then $g_i^s$ is indeed $(\alpha,\beta,\epsilon)$-correlated with $A_i$ (Lemma~\ref{cor:one-step}). Finally we prove that if $\tilA^s$ is $(2\delta,2)$-\near to $\trA$ then $\|\tilA^{s+1} - \trA\|\le 2\|\trA\|$ (Lemma~\ref{lem:nowhiten}). These lemmas together with Theorem~\ref{thm:tiger-convergence} imply Theorem~\ref{thm:simplestupdate-infinitesample}, the simplified version of Theorem~\ref{thm:simplestupdate} where the number of samples $p$ is assumed to be infinite (i.e. we have access to the true expectation $g^s$). In Appendix~\ref{sec:sample_complexity} we prove the sample complexity bounds we need and this completes the proof of Theorem~\ref{thm:simplestupdate}. 

 %This yields our main theorem. 
%
%In Section~\ref{sec:hypothesis} we complete the proof of our main theorem by proving various auxiliary lemmas we need in order to fit into the framework we developed in the previous section. In particular, in Section~\ref{sec:progress}, we prove that if $\tilA^s$ is $(2\delta,2)$-\near to $\trA$ then $g_i^s$ is indeed $(\alpha,\beta,\epsilon)$-correlated with $A_i$ (Lemma~\ref{lem:checking_condtion1}). Finally we prove that if $\tilA^s$ is $(2\delta,2)$-\near to $\trA$ then $\|\tilA^{s+1} - \trA\|\le 2\|\trA\|$ (Lemma~\ref{lem:nowhiten}). This yields our main theorem. 

%\Tnote{This only proves the infinite samples version of the Theorem}
 %that Algorithm~\ref{eqn:simplestupdate} converges to a globally optimal solution. In particular when we set $\eta = O(m/k)$ the algorithm converges to a point that has systematic bias $O(\epsilon/\alpha) = O(k/n)$. This yields our main theorem, that the neurally plausible update rule introduced at the beginning of this section converges to a globally optimal code under the distributional assumptions of our generative model.

%\subsection{Old draft}
%\input{indicatorinclude_old}

%% file: newinitinclude10.tex
%\vspace{-0.1in}
\subsection{Further Applications}
\label{sec:otherrules}
%\vspace{-0.1in}
Here we apply our framework to design and analyze further variants of alternating minimization. 
%\vspace{-0.1in}
\subsubsection{Revisiting Olshausen-Field}\label{sec:ofanalysis}

In this subsection we analyze a variant of the Olshausen-Field update rule. However there are quadratic error terms that arise in the expressions we derive for $g^s$ and bounding them is more challenging. We will also need to make (slightly) stronger assumptions on the distributional model that for distinct $i_1,i_2,i_3$ we have $q_{i_1,i_2,i_3} = O(k^3/m^3)$ where $q_{i_1,i_2,i_3} =  \Pr[i_1,i_2,i_3\in S]$. 

\begin{theorem}\label{thm:OFupdate}
Suppose that $\tilA^0$ is $(2\delta,2)$-\near to $\trA$ and that
$\eta = \Theta(m/k)$. There is a variant of Olshausen-Field (given in Algorithm~\ref{eqn:ofrule} in Appendix~\ref{sec:otherrulesapp}) for which at each step $s$
%Then if each update step in Algorithm~\ref{eqn:ofrule} uses $p = \widetilde{\Omega}(mk)$ fresh samples, 
%and for each column $i$ 
we have
$$\|\tilA^s - \trA\|_F^2\le (1-\tau)^s \|\tilA^0 - \trA\|_F^2 + O(mk^2/n^2)$$
for some $0 < \tau < 1/2$ and for any $s=1,2, ... , T$. In particular it converges to $\trA$ geometrically until the error in Frobenius norm is $O(\sqrt{m}k/n)$. 
\end{theorem}

%
%\begin{theorem}\label{thm:OFupdate}
%Suppose that $\tilA^0$ is $(2\delta,2)$-\near to $\trA$ and that
%$\eta = \Theta(m/k)$. There is a variant of Olshausen-Field (given in Algorithm~\ref{eqn:ofrule} given Appendix~\ref{sec:otherrulesapp}) for which at each step $s$
%%Then if each update step in Algorithm~\ref{eqn:ofrule} uses $p = \widetilde{\Omega}(mk)$ fresh samples, 
%and for each column $i$ we have
%$$\|\tilA_i^s - \trA_i\|^2\le (1-\tau)^s \|\tilA^0_i - \trA_i\|^2 + O(k^2/n^2)$$
%for some $0 < \tau < 1/2$ and for any $s=1,2, ... , T$. In particular it converges to $\trA$ geometrically until the column-wise error is $O(k/n)$. 
%\end{theorem}
%\vspace{-0.1in}
We defer the proof of the main theorem to Appendix~\ref{sec:otherrulesapp}. Currently it uses a projection step
(using convex programming) that may not be needed but the proof requires it.  

%\vspace{-0.1in}
\subsubsection{Removing the Systemic Error}\label{sec:remove}

In this subsection, we design and analyze a new update rule that converges geometrically until the column-wise error is $n^{-\omega(1)}$. The basic idea is to engineer a new decoding matrix that projects out the components along the column currently being updated. This has the effect of replacing a certain error term in Lemma~\ref{lem:OF-simplified-update} with another term that goes to zero as $\tilA$ gets closer to $\trA$ (the earlier rules we have analyzed do not have this property).

We will use $B^{(s,i)}$ to denote the decoding matrix used when updating the $i$th column in the $s$th step. Then we set $B^{(s,i)}_i = \tilA_i \mbox{ and } B^{(s,i)}_j = \mbox{Proj}_{\tilA_i^\perp}\tilA_j \mbox{ for $j\ne i$.}$ Note that $B^{(s,i)}_{-i}$ (i.e. $B^{(s, i)}$ with the $i$th column removed) is now orthogonal to $\tilA_i$. We will rely on this fact when we bound the error. We defer the proof of the main theorem to Appendix~\ref{sec:nonoiseruleapp}. 
%\vspace{-0.02in}
\begin{theorem}\label{thm:nonoiserule}
Suppose that $\tilA^0$ is $(2\delta,2)$-\near to $\trA$ and that
$\eta = \Theta(m/k)$. There is an algorithm (given in Algorithm~\ref{eqn:nonoiserule} given in Appendix~\ref{sec:nonoiseruleapp}) for which at each step $s$, we have
%uses $p = \widetilde{\Omega}(mk)$ fresh samples, we have
%Then if each update step in Algorithm~\ref{eqn:nonoiserule} uses $p = \widetilde{\Omega}(mk)$ fresh samples, we have
$$\|\tilA_i^s - \trA_i\|^2 \le (1-\tau)^s \|\tilA^0_i - \trA_i\|^2 + n^{-\omega(1)}$$
for some $0 < \tau < 1/2$ and for any $s=1,2, ... , T$. In particular it converges to $\trA$ geometrically until the column-wise error is $n^{-\omega(1)}$. 
\end{theorem}

%\vspace{-0.2in}

%% file: togetherincludeflat10.tex
\section{Analysis of the Neural Algorithm}
\label{sec:hypothesis}

%The goal of this section is to prove Lemma~\ref{lem:checking_condtion0}. We show that when $\tilA$ is appropriately initialized to be  $\delta$-close to the true $\trA$, and 
%update rules ensure that the solution inductively satisfies to be $(2\delta,2)$-near to $\trA$, which then imply the statement of Lemma~\ref{lem:checking_condtion0}.

%In the first subsection we compute the expectation in Equation (\ref{eqn:simplestupdate}) and give its form. In the next part we show when the current $\tilA^s$ is $(2\delta,2)$-\near to $A$, then the update direction $g^s$ $(\alpha,\beta,\epsilon)$-points to $\trA$. Finally we show that assuming $\tilA^s$ is $(\delta,2)$-close to $\trA$, so is $\tilA^{s+1}$. Based on these, an easy induction on $s$ completes the argument. 

In Lemma~\ref{lem:OF-simplified-update} we gave a new (and more useful) expression that describes the update direction under the assumptions of our generative model. Here we will make crucial use of Lemma~\ref{lem:OF-simplified-update} in order to prove that $g^s_i$ is $(\alpha,\beta,\epsilon)$-correlated with $A_i$ (Lemma~\ref{lem:checking_condtion1}). Moreover we use Lemma~\ref{lem:OF-simplified-update} again to show that $\|\tilA^{s+1} - \trA\|\le 2\|\trA\|$ (Lemma~\ref{lem:nowhiten}). Together, these auxiliary lemmas imply that the column-wise error decreases in the next step and moreover the errors across columns are uncorrelated. 

We assume that each iteration of Algorithm~\ref{eqn:simplestupdate} takes infinite number of samples, and prove the corresponding simplified version of Theorem~\ref{thm:simplestupdate}. The proof of this Theorem highlights the essential ideas of behind the proof of the Theorem~\ref{thm:simplestupdate}, which can be found at Appendix~\ref{sec:sample_complexity}. 

\begin{theorem}\label{thm:simplestupdate-infinitesample}
Suppose that $\tilA^0$ is $(2\delta,2)$-\near to $\trA$ and that
$\eta = \Theta(m/k)$. Then if each update step in Algorithm~\ref{eqn:simplestupdate} uses infinite number of samples at each iteration, we have
$$\|\tilA_i^s - \trA_i\|^2\le (1-\tau)^s \|\tilA^0_i - \trA_i\|^2 + O(k^2/n^2)$$
for some $0 < \tau < 1/2$ and for any $s=1,2, ... , T$. In particular it converges to $\trA$ geometrically until the column-wise error is $O(k/n)$. 
\end{theorem}

The proof is deferred to the end of this section.

\subsection{Making Progress}\label{sec:progress}

In Lemma~\ref{lem:OF-simplified-update} we showed that $g^s_i = p_i q_i ( \lambda_i\tilA_i^s - \trA_i +\epsilon^s_i + \gamma)$ where $\lambda_i = \inner{\tilA_i, \trA_i}$. Here we will prove that $g_i^s$ is $(\alpha,\beta,\epsilon)$-correlated with $\trA_i$. Recall that we fixed $\delta = O^*(1/\log n)$. The main intuition is that $g^s_i$ is mostly equal to $p_iq_i (\tilA_i^s - \trA_i)$ with a small error term.

\begin{lemma}
\label{lem:checking_condtion1}
\label{lem:closetocorrelated}
If a vector $g^s_i$ is equal to $4\alpha (\tilA_i^s - \trA_i) + v$ where $\|v\| \le \alpha \|\tilA_i^s - \trA_i\| + \zeta$, then $g^s_i$ is $(\alpha, 1/100\alpha, \zeta^2/\alpha)$-correlated with $\trA_i$, more specifically,
$$
\langle g^s_i, \tilA^s_i - \trA_i\rangle \ge  \alpha \|\tilA^s_i -\trA_i \|^2 + \frac{1}{100\alpha} \|g_i\|^2 - \zeta^2/\alpha.
$$
\end{lemma}

\iffalse
\begin{lemma}\label{lem:checking_condtion1} If $\tilA^s$ is $(2\delta,2)$-\near to $\trA$, then 
$$\langle g^s_i, \tilA^s_i - \trA_i\rangle \ge  \frac{1}{p_iq_i(4+2\delta)}\|g_i\|^2 + \frac{(p_iq_i)(1-\delta)}{2}\|\tilA_i -\trA_i \|^2 - O(k^2/mn).$$
\end{lemma}
\fi

\noindent In particular,  $g_i^s$ is 
$(\alpha,\beta,\epsilon)$-correlated with $\trA_i$, where  $\alpha  = \Omega(k/m)$, $\beta \ge  \Omega(m/k)$ and $\epsilon = O(k^3/mn^2)$. We can now apply Theorem~\ref{thm:tiger-convergence} and conclude that the column-wise error gets smaller in the next step:

\begin{corollary}\label{cor:one-step}
If $\tilA^s$ is $(2\delta,2)$-\near to $\trA$ and $\eta \le \min_i (p_iq_i(1-\delta)) = O(m/k)$, then $g_i^s= p_i q_i ( \lambda_i\tilA_i^s - \trA_i +\epsilon^s_i + \gamma)$ is $(\Omega(k/m), \Omega(m/k), O(k^3/mn^2))$-correlated with $\trA_i$, and further
$$ \| \tilA^{s+1}_i-\trA_i\|^2 \leq (1-2\alpha\eta) \|\tilA^s_i-\trA_i\|^2 +O(\eta k^2/n^2)$$
\end{corollary}

\begin{proof}[Proof of Lemma~\ref{lem:checking_condtion1}]
Throughout this proof $s$ is fixed and so we will omit the superscript $s$ to simplify notations. By the assumption, $g_i$ already has a component that is pointing to the correct direction $\tilA_i-\trA_i$, we only need to show that the norm of the extra term $v$ is small enough. First we can bound the norm of $g_i$ by triangle inequality: $\|g_i\| \le \|4\alpha(\tilA_i-\trA_i)\|+\|v\| \le 5\alpha\|(\tilA_i-\trA_i)\| + \zeta$, therefore $\|g_i\|^2 \le 50\alpha^2 \|(\tilA_i-\trA_i)\|^2 + 2\zeta^2$. Also, we can bound the inner-product between $g_i$ and $\tilA_i - \trA_i$ by $\langle g_i, \tilA_i - \trA_i\rangle \ge 4\alpha \|\tilA_i - \trA_i\|^2 - \|v\|\|\tilA_i - \trA_i\|$. 

Using these bounds, we will show $\langle g_i, \tilA_i - \trA_i\rangle -  \alpha \|\tilA_i -\trA_i \|^2 - \frac{1}{100\alpha} \|g_i\|^2 + \zeta^2/\alpha \ge 0$. Indeed we have
\begin{flalign*}
& \langle g_i, \tilA_i - \trA_i\rangle -  \alpha \|\tilA_i -\trA_i \|^2 - \frac{1}{100\alpha} \|g_i\|^2 + \zeta^2/\alpha  \\ &\qquad \ge \quad 4\alpha \|\tilA_i -\trA_i \|^2 - \|v\|\|\tilA_i - \trA_i\| - \alpha \|\tilA_i -\trA_i \|^2-\frac{1}{100\alpha} \|g_i\|^2 + \zeta^2/\alpha\\
&\qquad \ge \quad  3\alpha \|\tilA_i -\trA_i \|^2 - (\alpha \|\tilA_i - \trA_i\|+\zeta)\|\tilA_i - \trA_i\|-\frac{1}{100\alpha} (50\alpha^2\|(\tilA_i-\trA_i)\|^2+2\zeta^2) + \zeta^2/\alpha \\
&\qquad \ge \quad  \alpha \|\tilA_i -\trA_i \|^2 - \zeta \|\tilA_i - \trA_i\| + \frac{1}{4} \zeta^2/\alpha \\ 
&\qquad = \quad(\sqrt{\alpha}\|\tilA_i-\trA_i\|-\zeta/2\sqrt{\alpha})^2\ge 0.
\end{flalign*}
\noindent This completes the proof of the lemma.
\end{proof}

\begin{proof}[Proof of Corollary~\ref{cor:one-step}]
We use the form in Lemma~\ref{lem:OF-simplified-update}, $g^s_i = p_i q_i ( \lambda_i\tilA_i^s - \trA_i +\epsilon^s_i + \gamma)$ where $\lambda_i = \inner{\tilA_i, \trA_i}$. We can write $g^s_i = p_iq_i (\tilA_i^s-\trA_i) + p_iq_i((1-\lambda_i)\tilA_i^s + \epsilon^s_i + \gamma)$, so when applying Lemma~\ref{lem:closetocorrelated} we can use $4\alpha = p_iq_i = \Theta(k/m)$ and $v = p_iq_i((1-\lambda_i)\tilA_i^s + \epsilon^s_i + \gamma)$. The norm of $v$ can be bounded in two terms, the first term $p_iq_i(1-\lambda_i)\tilA_i^s$ has norm $p_iq_i (1-\lambda_i)$ which is smaller than $p_iq_i \|\tilA_i^s-\trA_i\|$, and the second term has norm bounded by $\zeta = O(k^2/mn)$.

By Lemma~\ref{lem:closetocorrelated} we know the vector $g^i_s$ is $(\Omega(k/m), \Omega(m/k), O(k^3/mn^2))$-correlated with $\tilA^s$. Then by Theorem~\ref{thm:tiger-convergence} we have the last part of the corollary.
\end{proof}

\subsection{Maintaining Nearness}\label{sec:nearness}

\begin{lemma}\label{lem:nowhiten}
Suppose that $\tilA^s$ is $(2\delta,2)$-\near to $\trA$. Then $\|\tilA^{s+1} - \trA\|\le 2\|\trA\|$ in Algorithm~\ref{eqn:simplestupdate}.
\end{lemma}

\begin{proof}
%First by Corollary~\ref{cor:one-step}, we have that 
%$$\|\tilA^{s+1}_i-\trA_i\| \le (1-\alpha\eta)\|\tilA^{s+1}_i-\trA_i\|  + O(k/n)\le 2(1-\alpha\eta)\delta + O(k/n) \le 2\delta$$
%First by Theorem~\ref{thm:making-progress} we know the new $\tilA^{s+1}$ is at least $(1-\gamma)^{(s+1)/2}\delta^0+O(\sqrt{k/n}) \le 2\delta_0$ close to $A$, so we only need 
%Therefore it suffices to prove that $\|\tilA^{s+1} - \trA\|\le 2\|\trA\|$.
As in the proof of the previous lemma, we will make crucial use of Lemma~\ref{lem:OF-simplified-update}. Substituting and rearranging terms we have:
\begin{eqnarray*}
\tilA^{s+1}_i - \trA_i &=&  \tilA_i^s - \trA_i - \eta g_i^s \\
& = & (1-\eta p_iq_i) (\tilA_i^s - \trA_i) + \eta p_iq_i (1-\lambda_i^s)\tilA_i^s - \eta p_i \Big( \tilA^s_{-i} \mbox{diag}(q_{i,j}) \left(\tilA_{-i}^s\right)^T \Big)\trA_i \pm \gamma
\end{eqnarray*}

Our first goal is to write this equation in a more convenient form. In particular let $U$ and $V$ be matrices such that $U_i = p_iq_i (1-\lambda_i^s)\tilA_i^s$ and $V_i = p_i \Big( \tilA^s_{-i} \mbox{diag}(q_{i,j}) \left(\tilA_{-i}^s\right)^T \Big)\trA_i$. Then we can re-write the above equation as:
$$
\tilA^{s+1}-\trA =  (\tilA^s - \trA)\mbox{diag}(1-\eta p_iq_i)+ \eta U - \eta V \pm \gamma
$$
where $\mbox{diag}(1-\eta p_iq_i)$ is the $m\times m$ diagonal matrix whose entries along the diagonal are $1-\eta p_iq_i$. 

We will bound the spectral norm of $\tilA^{s+1}-\tilA$ by bounding the spectral norm of each of the matrices of right hand side. The first two terms are straightforward to bound: $$\|(\tilA^s - \trA)\mbox{diag}(1-\eta p_iq_i)\| \le
\|\tilA^s - \trA\|\cdot (1-\eta \min_i p_iq_i) \le  2(1-\Omega(\eta k/m))\|\trA\| $$
where the last inequality uses the assumption that $p_i = \Theta(1)$ and $q_i \le O(k/m)$, and the assumption that $\|\tilA^s - \trA\|\le 2\|\trA\|$.

From the definition of $U$ it follows that $U = \tilA^s \mbox{diag}(p_iq_i(1-\lambda_i^s))$, and therefore 
$$\|U\| \le \delta \max_i p_iq_i \|\tilA^s\| = o(k/m)\cdot \|\trA\|$$
where we have used the fact that $\lambda_i^s \ge 1-\delta$ and $\delta = o(1)$, and $\|\tilA^s\| \le \|\tilA^s-\trA\|+\|\trA\| = O(\|\trA\|)$.

What remains is to bound the third term, and let us first introduce an auxiliary matrix $Q$ which we define as follows: $Q_{ii} = 0$ and $Q_{i,j} = q_{i,j}\inner{\tilA^s_i, \trA_i}$ for $i \neq j$. It is easy to verify that the following claim:
\begin{claim}
The $i$th column of $\tilA^sQ$ is equal to $\Big( \tilA^s_{-i} \mbox{diag}(q_{i,j}) \left(\tilA_{-i}^s\right)^T \Big)\trA_i$
\end{claim} 
Therefore we can write $V = \tilA^s Q\mbox{diag}(p_i)$. We will bound the spectral norm of $Q$ by bounding its Frobenus norm instead. Then from the definition of $A$, we have that:
 $$\|Q\|_F \le \left(\max_{i\neq j}q_{ij} \right) \sum_{i\neq j} \sqrt{\inner{\tilA^s_i,\trA_j}^2} = O(k^2/m^2)\|{\trA}^T\tilA^s\|_F$$
Moreover since ${\trA}^T\tilA^s$ is an $m\times m$ matrix, its Frobenius norm can be at most a $\sqrt{m}$ factor larger than its spectral norm. Hence we have
\begin{eqnarray*}
\|V\| \le \left(\max_i p_i\right)\|\tilA^s\| \|Q\| &\le& O(k^2\sqrt{m}/m^2)\|\tilA^s\|^2\|\trA\| \\
&\le & o(k/m)\|\trA\|
\end{eqnarray*}
where the last inequality uses the fact that $k = O(\sqrt{n}/\log n)$ and $\|\tilA^s\|\le O(\|\trA\|)$.  

%For the term $M_2$, its $i$-th column is equal to $\eta p_i \left(B_{-i}\mbox{diag}(q_{i,j})B^T_{-i}\right)A_i$. Let $Q\in \Re^{m\times m}$ be a matrix where $Q_{i,i} = 0$, and $Q_{i,j} = q_{i,j} \inner{B_j,A_i}$, then the term $M_2$ is
%$$
%M_2 = \mbox{diag}(\eta p_i) B Q.
%%$$
%To bound the spectral norm of $Q$, we use the fact that $q_{i,j}$'s are all small, and the Frobenius norm of $B^TA$ is at most $\sqrt{n}$ times its spectral norm (it is a matrix of rank at most $n$). Therefore:
%$$\|Q\|\le \|Q\|_F \le \max_{i\ne j} q_{i,j} \|B^TA\|_F \le \max_{i\ne j} q_{i,j} \sqrt{n} \|B^TA\| \le \max_{i\ne j} q_{i,j}6\|\tilA\|^2 = O(k^2\sqrt{n}/m^2)\|\tilA\|^2.$$
%Now the spectral norm of $M_2$ can be bounded by $\max_i\eta p_i \|B\|\|Q\| \le O(\eta k^2\sqrt{n}/m^2)\|\tilA\|^3 \le O(\eta k^2/m\sqrt{n})\|\tilA\| \le o(\eta k/m)\|\tilA\|$, where the steps are using the spectral norm bound on $A$ and the assumption that $k \le O(n/\log n)$.
%
%The term $M_3$ is negligible. Therefore we can bound the spectral norm of $\tilA^{s+1} - \trA$ as
Therefore, putting the pieces together we have: 
\begin{eqnarray*}
\|\tilA^{s+1} - \trA\| &\le & \|(\tilA^s - \trA)\mbox{diag}(1-\eta p_iq_i)\|+ \|\eta U\|+ \|\eta V\| \pm \gamma \\
&\le &  2(1-\Omega(\eta k/m))\|\tilA\|+ o(\eta k/m)\|\trA\| + o(\eta k/m)\|\trA\| \pm \gamma\\
&\le & 2\|\trA\|
\end{eqnarray*}
and this completes the proof of the lemma.
\end{proof}

\subsection{Proof of Theorem~\ref{thm:simplestupdate-infinitesample}}
We prove by induction on $s$. Our induction hypothesis is that the theorem is true at each step $s$ and $A^s$ is $(2\delta,2)$-near to $\trA$. The hypothesis is trivially true for $s = 0$. Now assuming the inductive hypothesis is true. Recall that Corollary~\ref{cor:one-step} of Section~\ref{sec:progress} says that if $\tilA^s$ is $(2\delta,2)$-\near to $\trA$, which is guaranteed by the inductive hypothesis, by then $g_i^s$ is indeed $(\Omega(k/m), \Omega(m/k), O(k^3/mn^2))$-correlated with $\trA_i$. Invoking our framework of analysis (Theorem~\ref{thm:tiger-convergence}),  we have that 
$$\|\tilA_i^{s+1} - \trA_i\|^2\le (1-\tau) \|\tilA^{s}_i - \trA_i\|^2 + O(k^2/n^2) \le (1-\tau)^{s+1}\|\tilA^0_i - \trA_i\|^2 + O(k^2/n^2) $$
Therefore it also follows that $\tilA^{s+1}$ is $2\delta$-close to $\trA$. Then we invoke Lemma~\ref{lem:nowhiten} to prove $\tilA^{s+1}$ has not too large spectral norm $\|\tilA^{s+1}-\trA\| \le 2\|\trA\|$, which completes the induction. 
%In Section~\ref{sec:hypothesis} we complete the proof of our main theorem by proving various auxiliary lemmas we need in order to fit into the framework we developed in the previous section. In particular, 
%, we proved that if . Finally we prove that if $\tilA^s$ is $(2\delta,2)$-\near to $\trA$ then $\|\tilA^{s+1} - \trA\|\le 2\|\trA\|$ (Lemma~\ref{lem:nowhiten}). This yields our main theorem. 
%$\|\tilA_i - \trA_i\| \le \delta$, $p = \Omega(m\log m)$, new samples $y^{(1)}, y^{(2)}, ..., y^{(p)}$
%\ENSURE Matrix $H$, $\|H_i - \trA_i\| \le \delta/2$
%\STATE Let $B = ${\sc Whiten}$(\tilA, O(\sqrt{m/n}))$.
%\FOR{$i = 1$ TO $m$}
%\STATE Let $B^{(i)}$ be a matrix whose columns $B^{(i)}_j = \Proj_{\tilA_i^\perp} B_j = (I-\tilA_i\tilA_i^T)B_j$
%\STATE Initialize $Q_i\in \R^{n\times n}$ to be the all zeros matrix
%\FOR{$j = 1$ TO $p$}
%\STATE Let $S$ be the support of $y^{(j)}$
%\IF{$i \in S$}
%\STATE Let $M_S$ be the projection onto the column span of $B^{(i)}_{S\backslash\{i\}}$
%\STATE Let $r^{(j)} = (I-M_S)y^{(j)}$ be the projection of $y^{(j)}$ onto the orthogonal space of $M_S$.
%\STATE {\bf if} $\delta < 1/100k$ and $\|r^{(j)}\|> k\delta$ {\bf then} $r^{(j)} = \vec{0}$.
%\STATE Let $Q_i = Q_i + r^{(j)}(r^{(j)})^T$.%(I-M_S)y^{(j)}(y^{(j)})^T(I-M_S)$.
%\ENDIF
%\ENDFOR
%\STATE Let $H_i$ be the top singular vector of $Q_i$
%\ENDFOR
%\RETURN $H$

%% file: init10.tex
\section{Initialization}\label{sec:init}

There is a large gap between theory and practice in terms of how to initialize alternating minimization. The usual approach is to set $\tilA$ randomly or to populate its columns with samples $y^{(i)}$. These often work
but we do not know how to analyse them. % there no known provable guarantees for these methods. 
Here we give a novel method for initialization which we show succeeds with very high probability. Our algorithm works by pairwise reweighting. Let $u = \trA\alpha$ and $v = \trA\alpha'$ be two samples from our model whose supports are $U$ and $V$ respectively. The main idea is that if we reweight fresh samples $y$ with a factor $\inner{y,u}\inner{y,v}$ and compute $$\widehat{M}_{u, v} = \frac{1}{p_2}  \sum_{i=1}^{p_2} \inner{y^{(i)},u}\inner{y^{(i)},v}y^{(i)}(y^{(i)})^T$$ then the top singular vectors will  correspond to columns $\trA_j$ where $j \in U \cap V$. (This is reminiscent of ideas in recent papers on dictionary learning, but more sample efficient.)

Throughout this section we will assume that the algorithm is given two sets of samples of size $p_1$ and $p_2$ respectively. Let $p = p_1 + p_2$. We use the first set of samples for the pairs $u, v$ that are used in reweighting and we use the second set to compute $\widehat{M}_{u, v}$ (that is, the same set of $p_2$ samples is used for each $u, v$ throughout the execution of the algorithm). Our main theorem is:

\begin{theorem}
\label{thm:initmain}
Suppose that Algorithm~\ref{alg:init} is given $p_1 = \widetilde{\Omega}(m)$ and $p_2 = \widetilde{\Omega}(mk)$ fresh samples and moreover $(a)$ $\trA$ is $\mu$-incoherent with $\mu= O^*(\frac{\sqrt{n}}{k \log^3 n})$, $(b)$ $m = O(n)$ and $(c)$ $\|\trA\| \leq O(\sqrt{\frac{m}{n}})$.
Then with high probability $\tilA$ is $(\delta, 2)$-near to $\trA$ where $\delta = O^*(1/\log n)$. 
\end{theorem}

\begin{algorithm}\caption{Pairwise Initialization}\label{alg:init}
\vspace{0.1in}

\textbf{Set} $L = \emptyset$

\textbf{While} $|L| < m$ choose samples $u$ and $v$

\vspace{0.1in}

$\qquad$ Set $\widehat{M}_{u, v} = \frac{1}{p_2}  \sum_{i=1}^{p_2} \inner{y^{(i)},u}\inner{y^{(i)},v}y^{(i)}(y^{(i)})^T$

\vspace{0.1in}

$\qquad$ Compute the top two singular values $\sigma_1,\sigma_2$ and top singular vector $z$ of $\widehat{M}_{u,v}$

\vspace{0.1in}

$\qquad$ \textbf{If} $\sigma_1 \ge \Omega(k/m)$ and $\sigma_2 < O^*(k/m\log m)$ 

\vspace{0.1in}

$\qquad$ $\qquad$ \textbf{If} $z$ is not within distance $1/\log m$ of any vector in $L$ (even after sign flip),  add $z$ to $L$

\vspace{0.1in}

\textbf{Set} $\widetilde{\tilA}$ such that its columns are $z \in L$ and output $A = \mbox{Proj}_{\mathcal{B}} \widetilde{\tilA}$ where $\mathcal{B}$ is the convex set defined in Definition~\ref{def:whiten}

\vspace{0.1in}
%\textbf{Set} $\widetilde{\tilA}$ such that its columns are $z \in L$ and $\tilA = \mbox{Proj}_{\mathcal{B}} \widetilde{\tilA}$ (where $\mathcal{B}$ is defined in Definition~\ref{def:whiten})

\vspace{0.1in}

\end{algorithm}

We will defer the proof of this theorem %and the following main lemma to 
to Appendix~\ref{sec:initapp}. The key idea is the following main lemma.  We will invoke this lemma several times in order to analyze Algorithm~\ref{alg:init} to verify whether or not the supports of $u$ and $v$ share a common element, and again to show that if they do we can approximately recover the corresponding column of $\trA$ from the top singular vector of $M_{u, v}$. 

\begin{lemma}\label{lem:initform}
Suppose $u = \trA\alpha$ and $v = \trA\alpha'$ are two random samples with supports $U$, $V$ respectively. Let $\beta = {\trA}^T u$ and $\beta' = {\trA}^T v$. Let $y = \trA\trx$ be random sample that is independent of $u$,$v$, then 
\begin{equation}
M_{u,v} := \E[\inner{u,y}\inner{v,y} yy^T] = \sum_{i\in U\cap V} q_ic_i\beta_i\beta'_i\trA_i{\trA_i}^T + E_1+E_2+E_3,\label{eq:initform}
\end{equation} where $q_i = \Pr[i\in S]$, $c_i = \E[{\trx_i}^4|i\ne S]$, and the error terms are:
\[
\begin{array}{rll}
E_1 & = & \sum_{i\not\in U\cap V} q_ic_i\beta_i\beta'_i \trA_i{\trA_i}^T\\
E_2 & = & \sum_{i,j\in[m],i\ne j} q_{i,j}\beta_i\beta_i' \trA_j{\trA_j}^T\\
E_3 & = & \sum_{i,j\in[m],i\ne j} q_{i,j}(\beta_i\trA_i \beta'_j{\trA_j}^T+\beta'_i\trA_i\beta_j {\trA_j}^T).
\end{array}
\]
Moreover the error terms $E_1+E_2+E_3$ has spectral norm bounded by $O^*(k/m\log n)$, $|\beta_i|\ge \Omega(1)$ for all $i\in \mbox{{\em supp}}(\alpha)$ and $|\beta'_i|\ge \Omega(1)$ for all $i\in \mbox{{\em supp}}(\alpha')$.
\end{lemma}

%\noindent We will invoke this lemma several times in order to analyze Algorithm~\ref{alg:init} to verify whether or not the supports of $u$ and $v$ share a common element, and again to show that if they do we can approximately recover the corresponding column of $\trA$ from the top singular vector of $M_{u, v}$. 

\begin{proof}[Proof of Lemma~\ref{lem:initform}]
	We will prove this lemma in three parts. First we compute the expectation and show it has the desired form. Recall that $y = \trA\trx$, and so:
	\begin{eqnarray*}
		M_{u,v} &=& \E_{S} \Big [\E_{\trx_S}[\inner{u, \trA_S \trx_S}\inner{v, \trA_S \trx_S}\trA_S\trx_S{\trx_S}^T{\trA_S}^T|S]\Big ]\\
		&=& \E_{S} \Big [\E_{\trx_S}[\inner{\beta,\trx_S}\inner{\beta',\trx_S}\trA_S\trx_S{\trx_S}^T{\trA_S}^T|S]\Big ] \\
		&=& \E_{S} \Big [ \sum_{i\in S} c_i\beta_i\beta'_i \trA_i {\trA_i}^T + \sum_{i,j\in S, i\ne j}\Big (\beta_i\beta_i' \trA_j{\trA_j}^T+\beta_i \beta'_j \trA_i {\trA_j}^T+\beta'_i \beta_j\trA_i {\trA_j}^T \Big ) \Big ] \\
		&=& \sum_{i\in [m]}q_ic_i\beta_i\beta'_i \trA_i {\trA_i}^T + \sum_{i,j\in [m], i\ne j}q_{i,j}\Big (\beta_i\beta_i' \trA_j{\trA_j}^T+\beta_i \beta'_j \trA_i {\trA_j}^T+\beta'_i \beta_j \trA_i {\trA_j}^T \Big )
	\end{eqnarray*}
	where the second-to-last line follows because  the entries in $\trx_S$ are independent and have mean zero, and the only non-zero terms come from ${\trx_i}^4$ (whose expectation is $c_i$) and ${\trx_i}^2{\trx_j}^2$ (whose expectation is one). Equation (\ref{eq:initform}) now follows by rearranging the terms in the last line. What remains is to bound the spectral norm of $E_1, E_2$ and $E_3$. 
	
	\vspace{0.5pc}

	Next we establish some useful properties of $\beta$ and $\beta'$:
	
	\begin{claim}
		\label{claim:betabound}
		With high probability it holds that $(a)$ for each $i$ we have $|\beta_i - \alpha_i| \le \frac{\mu k \log m} {\sqrt{n}}$ and $(b)$ $\|\beta\| \le O(\sqrt{mk/n})$.
	\end{claim}
	
	\noindent In particular since the difference between $\beta_i$ an $\alpha_i$ is $o(1)$ for our setting of parameters, we conclude that if $\alpha_i \neq 0$ then $C-o(1) \le |\beta_i| \le O(\log m)$  and if $\alpha_i = 0$ then $|\beta_i| \le \frac{\mu k \log m}{\sqrt{n}}$. 
	
	\begin{proof}
		Recall that $U$ is the support of $\alpha$ and let $R = U\backslash\{i\}$. Then:
		$$\beta_i -\alpha_i  =  {\trA_i}^T \trA_U \alpha_U - \alpha_i =  {\trA_i}^T \trA_R \alpha_R$$
		and since $\trA$ is incoherent we have that $\|{\trA_i}^T\trA_R\| \le \mu\sqrt{k/n}$. Moreover the entries in $\alpha_R$ are independent and subgaussian random variables, and so with high probability $|\inner{{\trA_i}^T \trA_R, \alpha_R}| \le \frac{\mu k \log m}{\sqrt{n}}$ and this implies the first part of the claim. 
		
		For the second part, we can bound $\|\beta\| \leq \|\trA\|\|\trA_U\|\|\alpha\|$.  Since $\alpha$ is a $k$-sparse vector with independent and subgaussian entries, with high probability $\|\alpha\| \le O(\sqrt{k})$ in which case it follows that $\|\beta\| \le O(\sqrt{mk/n})$.
	\end{proof}
	
	%{\sc Rong: Here we proved something stronger than we need. This would imply that the algorithm works even when $m$ is a bigger polynomial of $n$ (with some extra works in the claim below). Shall we do that?}
	
	Now we are ready to bound the error terms.
	
	\begin{claim}
		\label{claim:errorterms}
		With high probability each of the error terms $E_1, E_2$ and $E_3$ in (\ref{eq:initform}) has spectral norm bounded at most $O^*(k/m\log m)$.
	\end{claim}
	
	\begin{proof}
		Let $S = [m]\backslash (U\cap V)$, then $E_1 = \trA_S D_1 {\trA_S}^T$ where $D_1$ is a diagonal matrix whose entries are $q_ic_i\beta_i\beta'_i$. We first bound $\|D_1\|$. To this end, we can invoke the first part of Claim~\ref{claim:betabound} to conclude that $|\beta_i\beta'_i| \le \frac{\mu^2 k^2 \log^2 m}{n}$. Also  $q_ic_i = \Theta(k/m)$ and so $$\|D_1\| \le O\left(\frac{\mu^2 k^3\log^2 m}{m n}\right) = O\left(\frac{\mu k^2\log^2 n}{m\sqrt{n}}\right) = O^*(k/m\log m)$$ Finally $\|A_S\| \leq \|A\| \leq O(1)$ (where we have used the assumption that $m = O(n)$), and this yields the desired bound on $\|E_1\|$. 
		
		The second term $E_2$ is a sum of positive semidefinite matrices and we will make crucial use of this fact below:
		$$E_2 = \sum_{i\ne j} q_{i,j}\beta_i\beta_i' \trA_j{\trA_j}^T \preceq O(k^2/m^2) \Big (\sum_{i} \beta_i\beta_i' \Big) \Big(\sum_{j} \trA_j{\trA_j}^T\Big)  \preceq O(k^2/m^2) \|\beta\|\|\beta'\| \trA{\trA}^T.$$
		Here the first inequality follows by using bounds on $q_{i,j}$ and then completing the square.  The second inequality uses Cauchy-Schwartz. We can now  invoke the second part of Claim~\ref{claim:betabound} and conclude that $\|E_2\| \le O(k^2/m^2)\|\beta\|\|\beta'\| \|\trA\|^2 \le O^*(k/m\log m)$ (where we have used the assumption that $m = O(n)$). 
		
		For the third error term $E_3$, by symmetry we need only consider terms of the form $q_{i,j}\beta_i \beta_j' \trA_i{\trA_j}^T$. We can collect these terms and write them as $\trA Q{\trA}^T$, where $Q_{i,j} = 0$ if $i=j$ and $Q_{i,j} = q_{i,j}\beta_i\beta_j'$ if $i\ne j$. First, we bound the Frobenius norm of $Q$:
		$$
		\|Q\|_F = \sqrt{\sum_{i\ne j,i,j\in [m]}q_{i,j}^2\beta_i^2(\beta'_j)^2} \le \sqrt{O(k^4/m^4)(\sum_{i\in [m]}\beta_i^2)(\sum_{j\in[m]}(\beta'_j)^2)} \le O(k^2/m^2)\|\beta\|\|\beta'\|.
		$$
		Finally $\|E_3\| \le 2\|\trA\|^2\|Q\| \le O(m/n \cdot k^2/m^2)\|\beta\|\|\beta'\| \le O^*(k/m\log m)$, and this completes the proof of the claim. 
	\end{proof}
	
	The proof of the main lemma is now complete.
\end{proof}

%% file: supportflat.tex
%\section{Computing the Representation}
%
%Most popular methods for dictionary learning alternate between finding a representation of the samples (with respect to the current estimate for the dictionary) and using this representation to update the dictionary. Here we prove that indeed any dictionary $B$ that is column-wise close to $A$ and also has bounded spectral norm is good enough for us to recover both the support of each sample and a good approximation to its non-zero coefficients. 

\section{Threshold Decoding}
\label{subsec:support}

Here we show that
a simple thresholding method recovers the support of each sample with high probability (over the randomness of $\trx$). This corresponds to the fact that sparse recovery for incoherent dictionaries is much easier when the non-zero coefficients do not take on a wide range of values; in particular, one does not need iterative pursuit algorithms in this case. As usual let $y = \trA\trx$ be a sample from the model, and let $S$ be the support of $\trx$. 
Moreover suppose that $\trA$ is $\mu$-incoherent and let $\tilA$ be column-wise $\delta$-close to $A$.
Then

\begin{lemma}\label{thm:supp}
If $\frac{\mu}{\sqrt{n}} \le \frac{1}{2k}$ and $k  = \Omega^*(\log m)$ and $\delta = O^*(1/\sqrt{\log m}) $, then with high probability (over the choice of $\trx$) we have 
$ S= \{i~ \mbox{:}~| \langle \tilA_i, y \rangle |  > C/2\}$. Also for all $i\in S$ $\sgn(\inner{\tilA_i,y}) = \sgn(\trx_i)$.
\end{lemma}
%\Tnote{May define $O^*(1/\sqrt{\log m})$ to be $\kappa$.}

Consider $\langle \tilA_i, y \rangle = \langle \tilA_i, \trA_i \rangle \trx_i + Z_i$ where $Z_i = \sum_{j \neq i} \langle 
\tilA_i,  \trA_j\rangle \trx_j$ is a mean zero random variable which measures the contribution of the cross-terms. Note that $|\langle \tilA_i, \trA_i \rangle| \ge (1-\delta^2/2)$, so $|\langle \tilA_i, \trA_i \rangle \trx_i| $ is either larger than $(1-\delta^2/2)\xlbd$ or equal to zero depending on whether or not $i\in S$. Our main goal is to show that the variable $Z_i$ is much smaller than $C$ with high probability, and this follows by standard concentration bounds. 

\begin{proof}
Intuitively, $Z_i$ has two source of randomness: the support $S$ of $\trx$, and the random values of $\trx$ conditioned on the support. 
We prove a stronger statement that only requires second source of randomness. Namely, even conditioned on the support $S$, with high probability $S= \{i~ \mbox{:}~| \langle \tilA_i, y \rangle |  > \xlbd/2\}$.

%For any index $i\in [m]$, we know $\langle B_i, y \rangle = \langle B_i, A_i \rangle x_i + Z_i$ where $Z_i = \sum_{j \neq i} \langle B_i,  A_j\rangle x_j$. As explained before the decoding is correct at index $i$ whenever $Z_i$ is small. 
We remark that $Z_i$ is a sum of independent subgaussian random variables and the variance of $Z_i$ is equal to $\sum_{j\in S\backslash\{i\}} \langle \tilA_i,  \trA_j\rangle^2$. Next we bound each term in the sum as $$\langle \tilA_i,  \trA_j\rangle^2 \le 2(\langle \trA_i,  \trA_j\rangle^2+\langle \tilA_i-\trA_i,  \trA_j\rangle^2)
\le 2\mu^2 + 2\langle \tilA_i-\trA_i,  \trA_j\rangle^2.$$
On the other hand, we know $\|\trA_{S\backslash\{i\}}\| \le 2$ by Gershgorin's Disk Theorem. Therefore the second term can be bounded as $\sum_{j\in  S\backslash\{i\}}\langle \tilA_i-\trA_i,  \trA_j\rangle^2 = \|{\trA_{S\backslash\{i\}}}^T (\tilA_i-\trA_i)\|^2 \le O^*(1/\log m)$. Using this bound, we know the variance is at most $O^*(1/\log m)$:
$$\sum_{j\in S\backslash\{i\}} \langle \tilA_i,  \trA_j\rangle^2\le 2\mu^2 k + 2\sum_{j\in  S\backslash\{i\}}\langle \tilA_i-A_i,  \trA_j\rangle^2 \le O^*(1/\log m).$$ 
Hence we have that $Z_i$ is a subgaussian random variable with variance at most $O^*(1/\log m)$ and so we conclude that $Z_i\le C/4$ with high probability. Finally we can take a union bound over all indices $i \in [m]$ and this completes the proof of the lemma. 
%Now concentration property of subgaussian variables implies with high probability $Z_i\le C/4$. Hence the the decoding is correct at index $i$ with high probability. Taking the union bound over all indices gives the result.
\end{proof}

In fact, even if $k$ is much larger than $\sqrt{n}$, as long as the spectral norm of $\trA$ is small and the support of $\trx$ is random enough, the support recovery is still correct.

\begin{lemma}\label{thm:supp2}
If $k  = O(n/\log n)$, $\mu/\sqrt{n} < 1/\log^2 n$ and $\delta < O^*(1/\sqrt{\log m}) $, the support of $\trx$ is a uniformly $k$-sparse set, then with high probability (over the choice of $x$) we have 
$ S= \{i~ \mbox{:}~| \langle \tilA_i, y \rangle |  > \xlbd/2\}$. Also for all $i\in S$ $\sgn(\inner{\tilA_i,y}) = \sgn(\trx_i)$.
\end{lemma}

\noindent The proof of this lemma is very similar to the previous one. However, in the previous case we only used the randomness after conditioning on the support, but to prove this stronger lemma we need to use the randomness of the support.

First we will need the following elementary claim:

\begin{claim}\label{claim:spectral}
$\|{\trA}^T \tilA_i\| \leq O(\sqrt{m/n})$ and $|{\trA}^T_{j} \tilA_i| \leq O^*(1/\sqrt{\log m})$ for all $j\ne i$.
\end{claim}

\begin{proof}
The first part follows immediately from the assumption that $\trA$ and $\tilA$ are column-wise close and that $\|\trA\| = O(\sqrt{m/n})$. The second part follows %since the columns of $A$ are unit vectors. 
because $|{\trA}^T_{j} \tilA_i| \le |{\trA}^T_{j} \trA_i|+|{\trA}^T_{j} (\trA_i-\tilA_i)| \le O^*(1/\sqrt{\log m})$.
\end{proof}

\noindent Let $R = S\backslash\{i\}$. 
Recall that conditioned on choice of $S$, we have $\mbox{var}(Z_i) = \sum_{j \in R} \langle \tilA_i, \trA_j \rangle^2.$ We will bound this term with high probability over the choice of $R$. First we bound its expectation: 

\begin{lemma}\label{lemma:evar}
$\E_R[\sum_{j \in R}\langle \tilA_i, \trA_j \rangle^2] \le O(k/n)$
%\leq \frac{k \delta^2}{n}$
\end{lemma}

\begin{proof}
By assumption $R$ is a uniformly random subset of $[m]\setminus\{i\}$ of size $|R|$ (this is either $k$ or $k-1$). Then 
$$\E[ \sum_{j \in R} \langle \tilA_i - \trA_i, \trA_j \rangle^2] 
%\leq \E[ \sum_{j \in S} \langle \tilA_i, \trA_j \rangle^2]  
= \frac{|R|}{m-1} \|{\trA}^T \tilA_i\|^2 = O(k/n),$$ 
where the last step uses Claim~\ref{claim:spectral}.
\end{proof}

%Hence we have that the expected variance satisfies: $$\E_R[\mbox{var}(Z_i) ] \leq 2 \mu^2 k/n + O\Big (\frac{k \delta^2}{n}\Big )$$
%$\mu = \frac{1}{\sqrt{n}}$
% and since $k \leq \sqrt{n}/\mu$, we get $\Exp_R[\mbox{var}(Z_R) ] \leq \frac{4}{k}$. Our goal is to prove that $Z_i$ is close to zero with high probability, since in this case $\langle B_i, y \rangle$ will be close to  $\langle B_i, \trA_i \rangle x_i $ and the largest coordinates of $B^T y$ will indeed be the set $S$. 

However bounding the expected variance of $Z_i$ is not enough; we need a bound that holds with high probability over the choice of the support. Intuitively, we should expect to get bounds on the variance that hold with high probability because each term in the sum above (that bounds $\mbox{var}(Z_i) $) is itself at most $O^*(1/\log m)$, which easily implies  Theorem~\ref{thm:supp}.

\begin{lemma}\label{lem:varZR}
$\sum_{j \in R}\langle \tilA_i, \trA_j \rangle^2  \leq O^*(1/\log m)$ with high probability over the choice of $R$.
%With high probability over the choice of $S$, the variance $\mbox{var}(Z)  \leq \delta^2 \log n$. 
\end{lemma}

\begin{proof}
Let  $a_j = \langle \tilA_i, \trA_j \rangle^2$, then $a_j = O^*(1/\log m)$ and moreover $\sum_{j\neq i} a_j = O(m/n)$ using the same idea as in the proof of Lemma~\ref{lemma:evar}. Hence we can apply Chernoff bounds %Corollary~\ref{cor:chernoffforsupport} 
and conclude that with high probability $\sum_{j \neq i} a_j X_j = \sum_{j \in R}\langle \tilA_i, \trA_j \rangle^2  \leq O^*(1/\log m)$ where $X_j$ is an indicator variable for whether or not $j \in R$. 
%\Rnote{The Chernoff bound is nontrivial and is in the appendix of previous version.}
%$\bar{Z}_R := \sum_{j\neq i} \langle B_i - \trA_i, \trA_j \rangle^2\indicator{i\in R}$ be a sum of independent variables. Since each $ B_i - \trA_i, \trA_j \rangle^2$ is bounded by $O(\delta^2)$ and the $\Exp[\bar{Z}_R] < O(k/n \cdot \delta^2)$. By applying the Corollary~\ref{cor:chernoffforsupport} of Weighted Chernoff Bound, we have that with high probability $\bar{Z}_R \le O(\delta^2)$. Together with the fact that $\sum_{j\in R}\langle B_i, \trA_j \rangle^2 \le \mu^2k/n$, we get the bound for $\Exp[Z_R]$. 
\end{proof}

%\begin{lemma}\label{lem:Zi}
%With high probability over the choice of $x$, $Z_i =  \langle B_i, \sum_{j \neq i} x_j \trA_j\rangle$ has absolute value bounded by $O(\delta\sqrt{\log n})$. 
%\end{lemma}

%\begin{proof}
%By Lemma~\ref{lem:varZR}, with high probability over the choice of $R$, $var(Z_R)\le O(\delta^2)$. Conditioned on $R$, $Z_i$ is a sum of independent subgussian variables with variance $O(\delta^2)$. Thus with high probability over the choice of $x$, $Z_i \le O(\delta\sqrt{\log n})$. 
%\end{proof}

\begin{proof}[Proof of Lemma~\ref{thm:supp2}]
Using Lemma~\ref{lem:varZR} 
we have that with high probability over the choice of $R$, $var(Z_i) \leq O^*{1/\log m} $.  In particular, conditioned on the support $R$, $Z_i$ is the sum of independent subgaussian variables and so with high probability (using Theorem~\ref{thm:subgaussian}) $$|Z_i| \leq O(\sqrt{var{Z_i}\log n}) =O^*(1).$$ %Then using the assumptions that $\delta = O^*{1/\sqrt{\log m}}$ and that $k = O^*{\sqrt{n}/\mu \log n}$, we conclude that $$|Z_i| = |\langle B_i, y \rangle -  \langle B_i, \trA_i \rangle x_i|  = O^*{1}$$ 
Also as we saw before that $|\langle \tilA_i, \trA_i \rangle x_i |> (1-\delta^2/2)C$ if $i \in S$  and is zero otherwise. So we conclude that $|\langle \tilA_i, y\rangle| > C/2$ if and only if $i\in S$
which completes the proof. 
\end{proof}

\paragraph{Remark:} In the above lemma we only needs the support of $x$ satisfy concentration inequality in Lemma~\ref{lem:varZR}. This does not really require $S$ to be uniformly random.
%{\sc Rong: Actually, now that I think about this proof, it actually shows (with more complicated proof) that as long as the largest correlation ($\langle \trA_i,B_i\rangle$) is $\log m$ factor larger than the second largest correlation, the decoding is correct with high probability (of course we need to choose a different threshold). This can be true even when $B_i$ is very far from $A_i$.}

%% file: newinitincludeapp.tex
\section{More Alternating Minimization}

Here we prove Theorem~\ref{thm:OFupdate} and Theorem~\ref{thm:nonoiserule}. Note that in Algorithm~\ref{eqn:ofrule} and Algorithm~\ref{eqn:nonoiserule}, we use the expectation of the gradient over the samples instead of the empirical average. We can show that these algorithms would maintain the same guarantees if we used $p = \widetilde{\Omega}(mk)$ to estimate $g^s$ as we did in Algorithm~\ref{eqn:simplestupdate}. However these proofs would require repeating very similar calculations to those that we performed in Appendix~\ref{sec:sample_complexity}, and so we only claim that these algorithms maintain their guarantees if they use a polynomial number of samples to approximate the expectation. 

\subsection{Proof of Theorem~\ref{thm:OFupdate}}\label{sec:otherrulesapp}

We give a variant of the Olshausen-Field update rule in Algorithm~\ref{eqn:ofrule}. Our first goal is to prove that  each column of $g^s$ is $(\alpha,\beta,\epsilon)$-correlated with $\trA_i$. The main step is to prove an analogue of Lemma~\ref{lem:OF-simplified-update} that holds for the new update rule. %For simplicity, we also 

\begin{lemma}\label{lem:OFupdate}
Suppose that $\tilA^s$ is $(2\delta,5)$-\near to $\trA$ Then each column of $g^s$ in Algorithm~\ref{eqn:ofrule} takes the form
%$$\tilA_i^{s+1} = \tilA_i^s - \eta g_i^s,$$
%where 
$$g_i^s = q_i\left((\lambda_i^s)^2\tilA_i^s - \lambda_iA_i^s +\epsilon_i^s\right)$$
where $\lambda_i = \langle \tilA_i, \trA_i \rangle$.  Moreover the norm of $\epsilon_i^s$ can be bounded as $\|\epsilon_i^s\|\le O(k^2/mn)$. 
\end{lemma}

%\Tnote{This is only about infinite number of samples}

\noindent We remark that unlike the statement of Lemma~\ref{lem:checking_condtion1}, here we will not explicitly state the functional form of $\epsilon_i^s$ because we will not need it. 

\begin{proof} 
The proof parallels that of Lemma~\ref{lem:OF-simplified-update}, although we will use slightly different conditioning arguments as needed. Again, we define $\mathcal{F}_{\trx}$ as the event that $\sgn(\trx) = \sgn(\tilx)$, and let $\mathbf{1}_{\calF_{\trx}}$ be the indicator function of this event. We can invoke Lemma~\ref{thm:supp} and conclude that this event happens with high probability. Moreover let $\mathcal{F}_i$ be the event that $i$ is in the set $S = \mbox{supp}(\trx)$ and let $\mathbf{1}_{\mathcal{F}_i}$ be its indicator function. 

When event $\mathcal{F}_{\trx}$ happens, the decoding satisfies $\tilx_S = \tilA_S^T \trA_S \trx_S$ and all the other entries are zero. Throughout this proof $s$ is fixed and so we will omit the superscript $s$ for notational convenience. We can now rewrite $g_i$ as 
\begin{eqnarray*}
g_i &=&  \E[(y - \tilA \tilx) \tilx^T] =  \E[(y - \tilA \tilx) \tilx_i^T  \mathbf{1}_{\calF_{\trx}}  ] +  \E[(y - \tilA \tilx) \tilx_i^T (1- \mathbf{1}_{\calF_{\trx}}) ]   \\
&=& \E\left[(I-\tilA_S^T\tilA_S)\trA_S\trx_S{\trx_S}^T{\trA_S}^T\tilA_i \mathbf{1}_{\calF_{\trx}} \mathbf{1}_{\mathcal{F}_i} \right]  \pm \gamma \\
&=& \E\left[(I-\tilA_S^T\tilA_S)\trA_S\trx_S{\trx_S}^T{\trA_S}^T\tilA_i \mathbf{1}_{\mathcal{F}_i} \right] \\
%&& \qquad - \E\left[(I-\tilA_S^T\tilA_S)\trA_S\trx_S{\trx_S}^T{\trA_S}^T\tilA_i (1 - \mathbf{1}_{\calF_{\trx}} ) \mathbf{1}_{\mathcal{F}_i} \right]  \pm \gamma \\
&=& \E\left[(I-\tilA_S^T\tilA_S)\trA_S\trx_S{\trx_S}^T{\trA_S}^T\tilA_i \mathbf{1}_{\mathcal{F}_i} \right]  \pm \gamma
\end{eqnarray*}

\begin{algorithm}\caption{Olshausen-Field Update Rule}\label{eqn:ofrule}
\vspace{0.1in}
\textbf{Initialize} $\tilA^0$ that is $(\delta_0, 2)$-near to $\trA$

\textbf{Repeat}  for $s = 0, 1, ..., T$
\vspace{-0.1in}
\begin{align*}
\textbf{\textbf{Decode: }} & \tilx = \textrm{threshold}_{C/2}((\tilA^s)^Ty)  \textrm{ for each sample $y$}\nonumber\\
%\textrm{ for $i = 1, 2, ... , p$}\nonumber\\
\textbf{\textbf{Update: }} & \tilA^{s+1} = \tilA^s - \eta g^{s} \mbox{ where } g^{s} = \Exp[(y - \tilA^s \tilx) \tilx^T] \quad \nonumber\\
%= \frac{1}{p}\sum_{i=1}^p (y^{(i)} - \tilA^s \tilx^{(i)}) (\tilx^{(i)})^T \quad \nonumber\\
\textbf{\textbf{Project: }}& \tilA^{s+1} = \mbox{Proj}_{\mathcal{B}} \tilA^{s+1} \mbox{(where $\mathcal{B}$ is defined in Definition~\ref{def:whiten})}\nonumber
\end{align*}
\end{algorithm}

Once again our strategy is to rewrite the expectation above using  subconditioning where we first choose the support $S$ of $\trx$, and then we choose the nonzero values $\trx_S$. 
\begin{eqnarray*}
g_i &=& \E_S \Big[ \E_{\trx_S}[ (I-\tilA_S^T\tilA_S)\trA_S\trx_S{\trx_S}^T{\trA_S}^T\tilA_i \mathbf{1}_{\mathcal{F}_i} | S ] \Big] \pm \gamma \\
&=& \E\left[(I-\tilA_S\tilA_S^T)\trA_S{\trA_S}^T\tilA_i \mathbf{1}_{\mathcal{F}_i} \right] \pm \gamma\\
&=& \E\left[(I-\tilA_i\tilA_i^T - \tilA_R{\tilA_R}^T)(\trA_i{\trA_i}^T + \trA_R{\trA_R}^T)\tilA_i  \mathbf{1}_{\mathcal{F}_i} \right] \pm \gamma\\
&=& \E\left[(I-\tilA_i\tilA_i^T)(\trA_i{\trA_i}^T)\tilA_i  \mathbf{1}_{\mathcal{F}_i} \right]
+\E\left[(I-\tilA_i\tilA_i^T)\trA_R{\trA_R}^T\tilA_i  \mathbf{1}_{\mathcal{F}_i} \right]\\
&& \qquad - \E\left[\tilA_R{\tilA_R}^T \trA_i{\trA_i}^T\tilA_i \mathbf{1}_{\mathcal{F}_i} \right]
- \E\left[\tilA_R\tilA_R^T \trA_R {\trA_R}^T\tilA_i  \mathbf{1}_{\mathcal{F}_i} \right] \pm \gamma
\end{eqnarray*}

Next we will compute the expectation of each of the terms on the right hand side. This part of the proof will be somewhat more involved than the proof of Lemma~\ref{lem:OF-simplified-update}, because the terms above are quadratic instead of linear. The leading term is equal to $q_i (\lambda_i \trA_i - \lambda_i^2 \tilA_i)$ and the remaining terms contribute to $\epsilon_i$. The second term is equal to $(I-\tilA_i\tilA_i^T) \trA_{-i}\mbox{diag}(q_{i,j}){\trA_{-i}}^T \tilA_i$ which has spectral norm bounded by $O(k^2/mn)$. The third term is equal to $\lambda_i \tilA_{-i}\mbox{diag}(q_{i,j}){\trA_{-i}}^T \trA_i$ which again has spectral norm bounded by $O(k^2/mn)$.
The final term is equal to 
\begin{align*}
\E\left[\tilA_R\tilA_R^T \trA_R {\trA_R}^T\tilA_i \mathbf{1}_{\mathcal{F}_i}\right] & = \sum_{j_1,j_2\ne i} \E[(\tilA_{j_1}\tilA_{j_1}^T)(\trA_{j_2}{\trA_{j_2}}^T)\tilA_i \mathbf{1}_{\mathcal{F}_i} \mathbf{1}_{\mathcal{F}_{j_1}} \mathbf{1}_{\mathcal{F}_{j_2}}]\\
& = \sum_{j_1\ne i} \left(\sum_{j_2\ne i} q_{i,j_1,j_2} \langle \trA_{j_2},\tilA_i \rangle\langle \trA_{j_2},\tilA_{j_1} \rangle\right) \tilA_{j_1} \\
& = \tilA_{-i} v.
\end{align*}
where $v$ is a vector whose $j_2$-th component is equal to $\sum_{j_2\ne i} q_{i,j_1,j_2} \langle \trA_{j_2},\tilA_i \rangle\langle \trA_{j_2},\tilA_{j_1} \rangle$. The absolute value of $v_{j_2}$ is bounded by 
\begin{align*}
|v_{j_2}| & \le O(k^2/m^2) |\langle \trA_{j_2},\tilA_i \rangle| + O(k^3/m^3) (\sum_{j_2\ne j_1,i} (\langle \trA_{j_2},\tilA_i \rangle^2 + \langle \trA_{j_2},\tilA_{j_1} \rangle^2))\\
& \le O(k^2/m^2) |\langle \trA_{j_2},\tilA_i \rangle| + O(k^3/m^3)\|\trA\|^2 = O(k^2/m^2)(|\langle \trA_{j_2},\tilA_i \rangle|+k/n).
\end{align*}
The first inequality uses bounds for $q$'s and the AM-GM inequality, the second inequality uses the spectral norm of $\trA$. We can now bound the norm of $v$ as follows $$\|v\| \le O(k^2/m^2 \cdot \sqrt{m/n})$$ and this implies that the last term satisfies $\|\tilA_{-i}\|\|v\| \le O(k^2/mn)$. Combining all these bounds completes the proof of the lemma.
\end{proof}

We are now ready to prove that the update rule satisfies Definition~\ref{eqn:con}. This again uses Lemma~\ref{lem:closetocorrelated}, % The proof of this lemma is just a calculation and is in fact an identical calculation to the one given in the proof of Lemma~\ref{lem:checking_condtion1}, 
except that we invoke Lemma~\ref{lem:OFupdate} instead. Combining these lemmas we obtain:

\begin{lemma}\label{lem:OFtiger}
Suppose that $\tilA^s$ is $(2\delta,5)$-\near to $\trA$. Then for each $i$, $g^s_i$ as defined in Algorithm~\ref{eqn:ofrule} is $(\alpha,\beta,\epsilon)$-correlated with $\trA_i$, where $\alpha = \Omega(k/m)$, $\beta \ge \Omega(m/k)$ and $\epsilon = O(k^3/mn^2)$.
\end{lemma}

%\Tnote{infinite number of samples}

Notice that in the third step in Algorithm~\ref{eqn:ofrule} we project back (with respect to Frobenius norm of the matrices) into a convex set $\mathcal{B}$ which we define below. Viewed as minimizing a convex function with convex constraints, this projection can be computed by various convex optimization algorithm, e.g. subgradient method (see Theorem 3.2.3 of Section 3.2.4 of Nesterov's seminal Book~\cite{Nesterov} for more detail).  Without this modification, it seems that the update rule given in Algorithm~\ref{eqn:ofrule} does not necessarily preserve nearness. 

\begin{definition}\label{def:whiten}
Let $\mathcal{B} = \{A | A \mbox{ is } \delta_0 \mbox{ close to } A^0 \mbox{ and } \|A\| \leq 2 \|\trA \|\}$
\end{definition}

\noindent The crucial properties of this set are summarized in the following claim:

\begin{claim}
$(a)$ $\trA \in \mathcal{B}$ and $(b)$ for each $A \in \mathcal{B}$, $A$ is $(2\delta_0, 5)$-near to $\trA$
\end{claim}

\begin{proof}
The first part of the claim follows because by assumption $\trA$ is $\delta_0$-close to $A^0$ and $\|\trA - A^0\| \leq 2 \|\trA\|$. Also the second part follows because $\|A - \trA\| \leq \|A - A^0\| + \|A^0 - \trA\| \leq 4\|\trA\|$. This completes the proof of the claim.
\end{proof}

By the convexity of $\mathcal{B}$ and the fact that $\trA \in \mathcal{B}$, we have that projection doesn't increase the error in Frobenius norm.

\begin{claim}For any matrix $\tilA$, $\|\mbox{Proj}_{\mathcal{B}} \tilA - \trA \|_F \le \|\tilA - \trA\|_F$. 
\end{claim}

We now have the tools to analyze Algorithm~\ref{eqn:ofrule} by fitting it into the framework of Corollary~\ref{cor:tiger-convexity}. In particular, we prove that it converges to a globally optimal solution by connecting it to an approximate form of {\em projected} gradient descent:

\begin{proof} [Proof of Theorem~\ref{thm:OFupdate}]
We note that projecting into $\mathcal{B}$ ensures that at the start of each step $\|\tilA^s - \trA\| \leq 5 \|\trA\|$. Hence   $g^s_i$ is $(\Omega(k/m), \Omega(m/k), O(k^3/mn^2))$-correlated with $\trA_i$ for each $i$, which follows from  Lemma~\ref{lem:OFtiger}. This implies that $g^s$ is $(\Omega(k/m), \Omega(m/k), O(k^3/n^2))$-correlated with $\trA$ in Frobenius norm. Finally we can apply Corollary~\ref{cor:tiger-convexity} (on the matrices with Frobenius) to complete the proof of the theorem. 
\end{proof}

%Hence by Theorem~\ref{thm:tiger-convergence} after $O(m\log n/k\eta) = O^*(\log n)$ steps the $\tilA$ matrix satisfy $\|\tilA - A\|_F \le O(\sqrt{m/n})$. In particular, since the Frobenius norm of $\tilA-A$ is not going to be larger than $O(\sqrt{m/n})$ in future iterations, the constraint set $\mathcal{B}$ will not have any effect in later iterations.

%Now we can apply Theorem~\ref{thm:tiger-convergence} on each column, and conclude that each column converges after another $O(m\log n/k\eta) = O^*(\log n)$ number of iterations.

\subsection{Proof of Theorem~\ref{thm:nonoiserule}}\label{sec:nonoiseruleapp}

The proof of Theorem~\ref{thm:nonoiserule} is parallel to that of Theorem~\ref{thm:simplestupdate-infinitesample} and Theorem~\ref{thm:OFupdate}. As usual, our first step is to show that $g_s$ is correlated with $\trA$:

%where $\tilA^{(i)}$ is a slight twist of $\tilA$ that depends on the index $i$: 
%
%$$\tilA^{(i)}_i = \trA_i,  \quad \textrm{ and for $j\neq i$ }, \tilA^{(i)}_j  =  \Pi_{\tilA_i^{\bot}} \tilA_j = \tilA_j - \inner{\tilA_j,\tilA_i}\tilA_i/\|\tilA_i\|^2$$

%\noindent It turns out that given suitable starting conditions, this event happens with arbitrarily high probability and we will use this fact crucially in our analysis. 
%Throughout this section, we will use $\gamma$ to denote a vector whose norm can be taken to be any inverse polynomial. 

%\subsubsection{Functional Form of Update}

%Using similar analysis as Lemma~\ref{lem:OF-simplified-update}, we can show this update rule is correlated with $\trA$ with negligible error:

\begin{lemma}\label{lem:OFupdate-noerror}
Suppose that $\tilA^s$ is $(\delta,5)$-\near to $\trA$. Then for each $i$, $g^s_i$ as defined in Algorithm~\ref{eqn:nonoiserule} is $(\alpha,\beta,\epsilon)$-correlated with $\trA_i$, where $\alpha = \Omega(k/m)$, $\beta \ge \Omega(m/k)$ and $\epsilon \le n^{-\omega(1)}$.
\end{lemma}

\begin{proof} 
We chose to write the proof of Lemma~\ref{lem:OF-simplified-update} so that we can reuse the calculation here. In particular, instead of substituting $B$ for $\tilA^s$ in the calculation we can substitute $B^{(s,i)}$ instead and we get:
$$
g^{(s,i)} = p_iq_i (\lambda_i^s \tilA^s_i - \trA_i + B^{(s,i)}_{-i}\mbox{diag}(q_{i,j})B^{(s,i)T}_{-i} \trA_i) + \gamma.
$$
Recall that $\lambda^s_i = \langle \tilA_i^s, \trA_i\rangle $. Now we can write $g^{(s,i)} = p_iq_i (\tilA^s_i - \trA_i) + v$, where $$v = p_iq_i(\lambda_i^s -1)\tilA^s_i + p_iq_i B^{(s,i)}_{-i}\mbox{diag}(q_{i,j})B^{(s,i)T}_{-i} \trA_i + \gamma$$ Indeed the norm of the first term $ p_iq_i(\lambda_i^s -1)\tilA^s_i$ is smaller than $p_iq_i\|\tilA^s_i - \trA_i\|$.

Recall that the second term was the main contribution to the systemic error, when we analyzed earlier update rules. However in this case we can use the fact that $B^{(s,i)T}_{-i}\tilA^s_i = 0$ to rewrite the second term above as $$p_iq_i B^{(s,i)}_{-i}\mbox{diag}(q_{i,j})B^{(s,i)T}_{-i} (\trA_i-\tilA^s_i)$$
Hence we can bound the norm of the second term by $O(k^2/mn)\|\trA_i-\tilA^s_i\|$, which is also much smaller than $p_iq_i\|\tilA^s_i - \trA_i\|$. 

Combining these two bounds we have that $\|v\| \le p_iq_i\|\tilA^s_i - \trA_i\|/4 + \gamma$, so we can take $\zeta = \gamma = n^{-\omega(1)}$ in Lemma~\ref{lem:closetocorrelated}. We can complete the proof by invoking Lemma~\ref{lem:closetocorrelated} which implies that the $g^{(s, i)}$ is $(\Omega(k/m),\Omega(m/k), n^{-\omega(1)})$-correlated with $\tilA_i$.
\end{proof}

\begin{algorithm}\caption{Unbiased Update Rule}\label{eqn:nonoiserule}
\vspace{0.1in}
\textbf{Initialize} $\tilA^0$ that is $(\delta_0, 2)$-near to $\trA$

\textbf{Repeat} for $s = 0, 1, ..., T$
\vspace{-0.1in}
\begin{align*}
\textbf{\textbf{Decode: }}  &  x = \textrm{threshold}_{C/2}((\tilA^s)^Ty)  \textrm{ for each sample $y$}\nonumber\\
& \bar{x}^i = \textrm{threshold}_{C/2}((B^{(s,i)})^Ty)  \textrm{ for each sample $y$, and each $i\in [m]$}\nonumber\\
\textbf{\textbf{Update: }} &  \tilA^{s+1}_i = \tilA^s_i - \eta g^{s}_i \mbox{ where } g^{s}_i =  \Exp[ (y - {B^{(s,i)}}\bar{x}^i) \sgn(\tilx)_i^T ] \textrm{ for each $i\in [m]$}\nonumber
%\tilA^{s+1}_j = \tilA^s_j - \eta g^{s}_j \mbox{ where } g^{s}_j =  \sum_{i =1}^p (y - {B^{(s,j)}}\tilx^{(i)}) \sgn(\tilx^{(i)})_j^T \nonumber
\end{align*}
\end{algorithm}

This lemma would be all we would need, if we added a third step that projects onto $\mathcal{B}$ as we did in Algorithm~\ref{eqn:ofrule}. However here we do not need to project at all, because the update rule maintains nearness and thus we can avoid this computationally intensive step.

\begin{lemma}\label{lem:noerrornowhiten}
Suppose that $\tilA^s$ is $(\delta,2)$-\near to $\trA$. Then $\|\tilA^{s+1} - \trA\| \leq 2 \|\trA\|$ in Algorithm~\ref{eqn:nonoiserule}.
\end{lemma}

This proof of the above lemma parallels that of Lemma~\ref{lem:nowhiten}. We will focus on highlighting the differences in bounding the error term, to avoid repeating the same calculation. 

\begin{proof} [sketch]
We will use $\tilA$ to denote $\tilA^s$ and $B^{(i)}$ to denote $B^{(s,i)}$ to simplify the notation. 
Also let $\bar{A}_i$ be normalized so that $\bar{A}_i = \tilA_i/\|\tilA_i\|$ and then we can write $B^{(i)}_{-i} = (I-\bar{A}_i \bar{A}_i^T) \tilA_{-i}$. Hence the error term is given by $$(I-\bar{A}_i \bar{A}_i^T) \tilA_{-i} \mbox{diag}(q_{i,j}) \tilA_{-i}^T (I-\bar{A}_i \bar{A}_i^T) \trA_i$$ Let $C$ be a matrix whose columns are $C_i = (I-\bar{A}_i \bar{A}_i^T) \trA_i = A_i - \langle \bar{A}_i, \trA_i\rangle\bar{A}_i $. This implies that $\|C\| \le O(\sqrt{m/n})$. We can now rewrite the error term above as $$\tilA_{-i} \mbox{diag}(q_{i,j}) \tilA_{-i}^T C_i - (\bar{A}_i\bar{A}_i)^T\tilA_{-i}\mbox{diag}(q_{i,j}) \tilA_{-i}^T C_i$$

It follows from the proof of Lemma~\ref{lem:nowhiten} that the first term above has spectral norm bounded by $O(k/m \cdot \sqrt{m/n})$. This is because in Lemma~\ref{lem:nowhiten} we bounded the term $\tilA_{-i} \mbox{diag}(q_{i,j}) \tilA_{-i}^T \trA_i$ and in fact it is easily verified that all we used in that proof was the fact that $\|\trA\| = O(\sqrt{m/n})$, which also holds for $C$. 

All that remains is to bound the second term. We note that its columns are scalar multiples of $\bar{A}_i$, where the coefficient can be bounded as follows: $\|\bar{A}_i\|\|\tilA_{-i}\|^2 \|\mbox{diag}(q_{i,j})\| \|\trA_i\| \le O(k^2/mn)$. Hence we can bound the spectral norm of the second term iby $O(k^2/mn)\|\bar{A}\| = O^*(k/m\cdot \sqrt{m/n})$. We can now combine these two bounds, which together with the calculation in Lemma~\ref{lem:nowhiten} completes the proof. 
\end{proof}

These two lemmas directly imply Theorem~\ref{thm:nonoiserule}.

%% file: initapp10.tex
\section{Analysis of Initialization}\label{sec:initapp}

Here we prove an infinite sample version of Theorem~\ref{thm:initmain} by repeatedly invoking Lemma~\ref{lem:initform}. We give sample complexity bounds for it in Appendix~\ref{sec:init_sample_complexity} where we complete the proof of Theorem~\ref{thm:initmain}. 

\begin{theorem}\label{thm:initmain_infinite}Under the assumption of Theorem~\ref{thm:initmain}, if Algorithm~\ref{alg:init} has access to $M_{u,v}$ (defined in Lemma~\ref{lem:initform}) instead of the empirical average $\widehat{M}_{u,v}$, then with high probability $\tilA$ is $(\delta, 2)$-near to $\trA$ where $\delta = O^*(1/\log n)$. 
\end{theorem}
%That is, we assume Algorithm~\ref{alg:init} have access to $M_{u,v}$ instead of $\widehat{M}_{u,v}$. The fluctuation of $\widehat{M}_{u, v}$ around $M_{u,v}$ is given in Section~\ref{lem:initconcentration}, and in the proof of Theorem~\ref{thm:initmain}, we explained 

Our first step is to use Lemma~\ref{lem:initform} to show that when $u$ and $v$ share a unique dictionary element, there is only one large term in $M_{u,v}$ and the error terms are small. Hence the top singular vector of $M_{u,v}$ must be close to the corresponding dictionary element $A_i$.

\begin{lemma}
\label{lem:init}
Under the assumptions of Theorem~\ref{thm:initmain},
suppose $u = \trA\alpha$ and $v = \trA\alpha'$ are two random samples with supports $U$, $V$ respectively. When $U\cap V = \{i\}$ the top singular vector of $M_{u,v}$ is $O^*(1/\log n)$-close to $\trA_i$.
\end{lemma}

%Finally, Lemma~\ref{lem:init} follows from Wedin's Theorem:
\begin{proof} %[Proof of Lemma~\ref{lem:init}]
When $u$ and $v$ share a unique dictionary element $i$, the contribution of the first term in (\ref{eq:initform}) is just $q_ic_i\beta_i\beta'_i \trA_i{\trA_i}^T$. Moreover the coefficient $q_ic_i\beta_i\beta'_i$ is at least $\Omega(k/m)$ which follows from Lemma~\ref{lem:initform} and from the assumptions that $c_i\ge 1$ and $q_i =\Omega(k/m)$. 

On the other hand, the error terms are bounded by $\|E_1+E_2+E_3\|\le O^*(k/m\log m)$ which again by Lemma~\ref{lem:initform}. We can now apply Wedin's Theorem (see e.g. \cite{HJ}) to $$M_{u,v} = q_ic_i\beta_i\beta'_i\trA_i{\trA_i}^T + \underbrace{(E_1+E_2+E_3)}_{\mbox{perturbation}}$$ and conclude that its top singular vector must be $O^*(k/m\log m)/\Omega(k/m) = O^*(1/\log m)$-close to $\trA_i$, and this completes the proof of the lemma. 
\end{proof}

Using (\ref{eq:initform}) again, we can verify whether or not the supports of $u$ and $v$  share a unique element.

\begin{lemma}
\label{lem:initverify}
Suppose $u = \trA\alpha$ and $v = \trA\alpha'$ are two random samples with supports $U$, $V$ respectively. Under the assumption of Theorem~\ref{thm:initmain}, if the top singular value of $M_{u,v}$ is at least $\Omega(k/m)$ and the second largest singular value is at most $O^*(k/m\log m)$, then with high probability $u$ and $v$ share a unique dictionary element. 
\end{lemma}

\begin{proof}
By Lemma~\ref{lem:initform} we know with high probability the error terms have spectral norm $O^*(k/m\log m)$. Here we show when that happens, and the top singular value is at least $\Omega(k/m)$, second largest singular value is at most $O^*(k/m\log m)$, then $u$ and $v$ must share a unique dictionary element.

If $u$ and $v$ share no dictionary element, then the main part in Equation (\ref{eq:initform}) is empty, and the error term has spectral norm $O^*(k/m\log m)$. In this case the top singular value of $M_{u,v}$ cannot be as large as $\Omega(k/m)$.

If $u$ and $v$ share more than one dictionary element, there are more than one terms in the main part of (\ref{eq:initform}). Let $S = U\cap V$, we know $M_{u,v} = \trA_S D_S {\trA_S}^T+E_1+E_2+E_3$ where $D_S$ is a diagonal matrix whose entries are equal to $q_ic_i\beta_i\beta'_i$. All diagonal entries in $D_S$ have magnitude at least $\Omega(k/m)$. By incoherence we know $\trA_S$ have smallest singular value at least $1/2$, therefore the second largest singular value of $\trA_SD_S{\trA_S}^T$ is at least:
$$
\sigma_2(\trA_SD_S{\trA_S}^T) \ge \sigma_{min}(\trA_S)^2\sigma_2(D_S) \ge \Omega(k/m).
$$
Finally by Weyl's theorem (see e.g. \cite{HJ}) we know $\sigma_2(M_{u,v}) \ge \sigma_2(\trA_SD_S{\trA_S}^T) - \|E_1+E_2+E_3\| \ge \Omega(k/m)$. Therefore in this case the second largest singular value cannot be as small as $O^*(k/m\log m)$.

Combining the above two cases, we know when the top two singular values satisfy the conditions in the lemma, and the error terms are small, $u$ and $v$ share a unique dictionary element.
\end{proof}

Finally, we are ready to prove Theorem~\ref{thm:initmain}. The idea is every vector added to the list $L$ will be close to one of the dictionary elements (by Lemma~\ref{lem:initverify}), and for every dictionary element the list $L$ contains at least one close vector because we have enough random samples.

\begin{proof}[Proof of Theorem~\ref{thm:initmain_infinite}]
By Lemma~\ref{lem:initverify} we know every vector added into $L$ must be close to one of the dictionary elements.
On the other hand, for any dictionary element $\trA_i$, by the bounded moment condition of $\mathcal{D}$ we know 
\[
\begin{array}{rl}
\Pr[|U\cap V| = \{i\}] & = \Pr[i\in U]\Pr[i\in V] \Pr[(U\cap V)\backslash\{i\} =\emptyset|i\in U, j\in U] \\
&\ge \Pr[i\in U]\Pr[i\in V](1 - \sum_{j\ne i, j\in [m]} \Pr[j\in U\cap V|i\in U,j\in V])\\
& = \Omega(k^2/m^2)\cdot (1-m\cdot O(k^2/m^2)) \\&= \Omega(k^2/m^2).
\end{array}
\]
Here the inequality uses union bound. Therefore given $O(m^2\log^2 n/k^2)$ trials, with high probability there is a pair of $u$,$v$ that intersect uniquely at $i$ for all $i\in [m]$. By Lemma~\ref{lem:init} this implies there must be at least one vector that is close to $\trA_i$ for all dictionary elements.

Finally, since all the dictionary elements have distance at least $1/2$ (by incoherence), the connected components in $L$ correctly identifies different dictionary elements. The output $\tilA$ must be $O^*(1/\log m)$ close to $\trA$.
\end{proof}

%We now come to the proof of the main lemma:

%% file: samplecomplex10.tex
\section{Sample Complexity}\label{sec:sample_complexity}

In the previous sections, we analyzed various update rules assuming that the algorithm was given the exact expectation of some matrix-valued random variable. Here we show that these algorithms can just as well use approximations to the expectation (computed by taking a small number of samples). We will focus on analyzing the sample complexity of Algorithm~\ref{eqn:simplestupdate}, but a similar analysis extends to the other update rules as well. 

\subsection{Generalizing the $(\alpha,\beta,\epsilon)$-correlated Condition}

We first give a generalization of the framework we presented in Section~\ref{sec:schema} that handles random update direction $g^s$. 

\begin{definition}\label{eqn:whptiger}
A random vector $g^s$ is {\em $(\alpha, \beta,\epsilon_s)$-correlated-whp} with a desired solution $z^*$ if with probability at least $1-n^{-\omega(1)}$,
$$ \langle g^s, z^s-z^*\rangle \ge \alpha \|z^s-z^*\|^2 + \beta \|g^s\|^2 - \epsilon_s.$$
\end{definition}

This is a strong condition as it requires the random vector is well-correlated with the desired solution with very high probability. In some cases we can further relax the definition as the following:

\begin{definition}\label{eqn:randomtiger}
A random vector $g^s$ is {\em $(\alpha, \beta,\epsilon_s)$-correlated-in-expectation} with a desired solution $z^*$ if
$$ \E[\langle g^s, z^s-z^*\rangle] \ge \alpha \|z^s-z^*\|^2 + \beta\E[\|g^s\|^2] - \epsilon_s.$$
\end{definition}

We remark that $\E[\|g^s\|^2]$ can be much larger than $\|\E[g^s]\|^2$, and so the above notion is still stronger than requiring (say) that the expected vector $\E[g^s]$ is $(\alpha, \beta,\epsilon_s)$-correlated with $z^*$. 

\begin{theorem}\label{thm:random-tiger-convergence}
Suppose random vector $g^s$ is $(\alpha,\beta,\epsilon_s)$-correlated-whp with $z^*$ for $s = 1,2,\dots,T$ where $T \le \poly(n)$, and $\eta$ satisfies  $0< \eta \le 2\beta$, then  for any $s = 1,\dots,T$, 
$$ \E[\| z^{s+1}-z^*\|^2] \leq (1-2\alpha\eta) \|z^s-z^*\|^2 + 2\eta \epsilon_s$$
In particular, if $\|z^0-z^*\| \le \delta_0$ and $\epsilon_s \le \alpha \cdot o((1-2\alpha\eta)^s)\delta_0^2 + \epsilon$, 
 then the updates converge to $z^*$ geometrically with systematic error $\epsilon/\alpha$ in the sense that 
%$$\|z^s-z^*\|^2 \le (1-2\alpha\eta)^s \|z^0-z^*\| + \epsilon/\alpha.$$
$$\E[\|z^s-z^*\|^2] \le (1-2\alpha\eta)^s \delta_0^2 + \epsilon/\alpha.$$
%If $g^s$ is in addition $(\alpha,\beta,\epsilon_s)$-correlated-whp, then we know for $s \le \poly(n)$, with high probability
%$$\|z^s-z^*\|^2 \le (1-2\alpha\eta)^s \|z^0-z^*\|^2 + \epsilon/\alpha.$$
%$$\|x_s-x^*\|^2 \le (1-4\alpha\beta)^s \|z^0-z^*\| + \epsilon/(\alpha).$$
%Furthermore, if $\epsilon_s < \frac{\alpha}{2} \|z^s-z^*\|^2$ for $s = 1,\dots,T$, then 
%$$\E[\|z^s-z^*\|^2] \le (1-\alpha\eta)^s \|z^0-z^*\|^2.$$
%$$\|x_s-x^*\|^2 \le (1-2\alpha\beta)^s R^2$$
\end{theorem}

\noindent The proof is identical to that of Theorem~\ref{thm:tiger-convergence} except that we take the expectation of both sides.  
\subsection{Proof of Theorem~\ref{thm:simplestupdate}}

In order to prove Theorem~\ref{thm:simplestupdate}, we proceed in two steps. First we show when $\tilA^s$ is $(\delta_s, 2)$-\near to $\trA$, the approximate gradient is $(\alpha,\beta,\epsilon_s)$-correlated-whp with optimal solution $\trA$, with $\epsilon_s \le O(k^2/mn)+\alpha \cdot o(\delta_s^2)$. This allows us to use Theorem~\ref{thm:random-tiger-convergence} as long as we can guarantee the spectral norm of $\tilA^s-\trA$ is small. Next we show a version of Lemma~\ref{lem:nowhiten} which works even with the random approximate gradient, hence the nearness property is preserved during the iterations. These two steps are formalized in the following two lemmas, and we defer the proofs until the end of the section.

%\begin{corollary}\label{cor:one-step}
%If $\tilA^s$ is $(2\delta,2)$-\near to $\trA$ and $\eta \le \min_i (p_iq_i(1-\delta)) = O(m/k)$, then $\tilA^s$ is $(\Omega(k/m), \Omega(m/k), O(k^3/mn^2))$-correlated with $\tilA^s$, and further
%$$ \| \tilA^{s+1}_i-\tilA_i\|^2 \leq (1-2\alpha\eta) \|\tilA^s_i-\tilA_i\|^2 +O(\eta k^2/n^2)$$
%\end{corollary}

\begin{lemma}
\label{lem:stepconcentration}
Suppose $\tilA^s$ is $(2\delta,2)$-\near to $\trA$ and $\eta \le \min_i (p_iq_i(1-\delta)) = O(m/k)$, then $\widehat{g}^s_i$ as defined in Algorithm~\ref{eqn:simplestupdate} is $(\alpha, \beta, \epsilon_s)$-correlated-whp with $\trA_i$ with $\alpha = \Omega(k/m)$, $\beta = \Omega(m/k)$ and $\epsilon_s \le \alpha \cdot  o(\delta_s^2) + O(k^2/mn)$.
\end{lemma}

%The reason we need to use $\delta_s$ instead of $\|\tilA^s_i-\trA_i\|$ is because different columns may converge at different speeds. 

%\begin{lemma}
%\label{lem:stepconcentration}
%Suppose $\tilA^s$ is $(\delta,2)$-\near to $\trA$ and $\eta \le \min_i (p_iq_i(1-\delta)) = O(m/k)$, then Let $g_i = \sum_{j=1}^p (y^{(j)} - \tilA^s \trxs{j}) \sgn({\trx_i}^{(j)})$, and let $w_i$ be the number of samples $y^{(j)}(j\in[p])$ that have $i\in S$. If $p = \widetilde{\Omega}(m)$ then with high probability $w_i = \widetilde{\Omega}(k)$ and $\|g_i/w_i - \E[(y - \tilA^s \wtilx|i\in S]\| \le (o(\delta_s)+O(\sqrt{k/n}))$. In particular, this implies the vector is $(\alpha, \beta, \epsilon_s)$-correlated-whp with $A_i$ with $\alpha = \Omega(k/m)$, $\beta = \Omega(m/k)$ and $\epsilon_s \le \alpha \cdot  o(\delta_s^2) + O(k^2/mn)$.
%\end{lemma}

%The proof of this lemma treats $g_i$ as sum of independent random vectors, and use vector Bernstein's inequality to give a concentration bound. 

\begin{lemma}\label{lem:nearnessconcentration}
Suppose $\tilA^s$ is $(\delta_s, 2)$-\near to $\trA$ with $\delta_s = O^*(1/\log n)$, and number of samples used in step $s$ is $p = \widetilde{\Omega}(mk)$, then with high probability $\tilA^{s+1}$ satisfies $\|\tilA^{s+1}-\trA\| \le 2\|\trA\|$.
\end{lemma}

We will prove these lemmas by bounding the difference between  $\widehat{g}_i^s$  and $g_i^s$ using various concentration inequalities. For example, we will use the fact that $\widehat{g}_i^s$ is close to $g_i^s$ in Euclidean distance.  % = \frac{1}{p}\cdot\sum_{i=1}^p (y^{(i)} - \tilA^s \tilx^{(i)}) \sgn(\tilx^{(i)})^T$ with $\Exp\left[(y - \tilA^s \tilx) \sgn(\tilx)^T\right]$ by concentration inequality. 

\begin{lemma}\label{lem:g_pertubation}
Suppose $\tilA^s$ is $(\delta_s, 2)$-\near to $\trA$ with $\delta_s = O^*(1/\log n)$, and number of samples used in step $s$ is $p = \widetilde{\Omega}(mk)$, then with high probability $\|\widehat{g}_i^s - g_i^s\| \le O(k/m)\cdot (o(\delta_s)+O(\sqrt{k/n}))$. 
\end{lemma}

Using the above lemma,  Lemma~\ref{lem:stepconcentration} now follows the fact that $g_i^s$ is correlated with $\trA_i$. %then we have $\widehat{g}_i^s$ is correlated-whp with $\trA_i$ as well. 
The proof of Lemma~\ref{lem:nearnessconcentration} mainly involves using matrix Bernstein's inequality to bound the fluctuation of the spectral norm of $A^{s+1}$.

\begin{proof}[Proof of Theorem~\ref{thm:simplestupdate}]
The theorem now follows immediately by combining Lemma~\ref{lem:stepconcentration} and Lemma~\ref{lem:nearnessconcentration}, and then applying Theorem~\ref{thm:random-tiger-convergence}. 
\end{proof}

\subsection{Sample Complexity for Algorithm~\ref{alg:init}}\label{sec:init_sample_complexity}

For the initialization procedure, when computing the reweighted covariance matrix $M_{u,v}$ we can only take the empirical average over samples. Here we show with only $\widetilde{\Omega}(mk)$ samples, the difference between the true $M_{u,v}$ matrix and the estimated $M_{u,v}$ matrix is already small enough.

\begin{lemma}
\label{lem:initconcentration}
In Algorithm~\ref{alg:init}, if $p = \widetilde{\Omega}(mk)$ then with high probability for any pair $u,v$ consider by Algorithm~\ref{alg:init}, we have  $\|M_{u,v} - \widehat{M}_{u,v}\| \le O^*(k/m\log n)$. 
\end{lemma}

The proof of this Lemma is deferred to Section~\ref{sec:initconcentration}.  notice that although in Algorithm~\ref{alg:init}, we need to estimate $M_{u,v}$ for many pairs $u$ and $v$, the samples used for different pairs do not need to be independent. Therefore we can partition the data into two parts, use the first part to sample pairs $u,v$, and use the second part to estimate $M_{u,v}$. In this way, we know that for each pair $u,v$ the whole initialization algorithm also takes $\widetilde{\Omega}(mk)$ samples. Now we are ready to prove Theorem~\ref{thm:initmain}.

\begin{proof}[Proof of Theorem~\ref{thm:initmain}]
First of all, the conclusion of Lemma~\ref{lem:init} is still true for $\widehat{M}_{u,v}$ when $p = \widetilde{\Omega}(mk)$. To see this, we could simply write $$\widehat{M}_{u,v} = q_ic_i\beta_i\beta'_i\trA_i{\trA_i}^T + \underbrace{(E_1+E_2+E_3) + (\widehat{M}_{u,v}  - M_{u,v})}_{\mbox{perturbation}} $$where $E_1,E_2,E_3$ are the same as the proof of Lemma~\ref{lem:init}. We can now view $\widehat{M}_{u,v}  - M_{u,v}$ as an additional perturbation term with the same magnitude. We have that when $U\cap V = \{i\}$ the top singular vector of $M_{u,v}$ is $O^*(1/\log n)$-close to $\trA_i$. Similarly, we can prove the conclusion of Lemma~\ref{lem:initverify} is also true for $\widehat{M}_{u,v}$. Note that we actually choose $p$ such that the perturbation of $\widehat{M}_{u,v}$ matches noise level in Lemma~\ref{lem:initverify}. 
Finally, the proof of the theorem follows exactly that of the infinite sample case given in Theorem~\ref{thm:initmain_infinite}, except that we invoke the finite sample counterparts of Lemma~\ref{lem:initform} and Lemma~\ref{lem:initverify} that we gave above.
\end{proof}

\subsection{Proofs of Auxiliary Lemmas}

Here we prove Lemma~\ref{lem:stepconcentration}, Lemma~\ref{lem:nearnessconcentration}, and Lemma~\ref{lem:initconcentration} which will follow from various versions of the Bernstein inequality. We first recall Bernstein's inequality that we are going to use several times in this section.  Let $Z$ be a random variable (which could be a vector or a matrix) chosen from some distribution $\mathcal{D}$ and let $Z^{(1)}, Z^{(2)}, ..., Z^{(p)}$ be $p$ independent and identically distributed samples from $\mathcal{D}$.  Bernstein's inequality implies that if $\Exp[Z] = 0$ and for each $j$, $\|Z^{(j)}\|\le R$ almost surely and $\E[(Z^{(j)})^2] \le \sigma^2$, then 
\begin{eqnarray}
\frac{1}{p}\left\|\sum_{i=1}^p Z^{(i)}\right\| \leq \widetilde{O}\left(\frac{R}{p} + \sqrt{\frac{\sigma^2}{p}}\right) \label{eqn:bernstein}
\end{eqnarray} 
 with high probability. The proofs below will involve computing good bounds on $R$ and $\sigma^2$. However in our setting, the random variables will not be bounded almost surely. We will use the following technical lemma to handle this issue. 

\begin{lemma}\label{lem:general_bernstein}
Suppose that the distribution of $Z$ satisfies $\Pr[\|Z\| \ge R (\log (1/\rho))^C] \le 1-\rho$ for some constant $C > 0$, then 
\begin{itemize}

\item[(a)] If $p = n^{O(1)}$ then $\|Z^{(j)}\| \leq \widetilde{O}(R)$ holds for each $j$ with high probability and 

\item[(b)] $\|\E[Z \mathbf{1}_{\|Z\| \ge \widetilde{\Omega}(R)}]\| = n^{-\omega(1)}$. 
\end{itemize}
\end{lemma}

\noindent In particular, if $\frac{1}{p} \sum_{j=1}^p Z^{(j)} (1-\mathbf{1}_{\|Z^{(j)}\| \ge \widetilde{\Omega}(R)})$ is concentrated with high probability, then $\frac{1}{p}\sum_{j=1}^p Z^{(j)}$ is too.

\begin{proof}
The first part of the lemma follows from choosing $\rho = n^{-\log n}$ and applying a union bound. The second part of the lemma follows from 
\begin{align*}
\E[Z\mathbf{1}_{\|Z\| \ge R\log^{2c} n}]& \le \E[\|Z\|\mathbf{1}_{\|Z\| \ge R\log^{2c} n}]\\
& = R\log^{2c}n\Pr[\|Z\|\ge R\log^{2c}n] + \int_{R\log^{2c}n}^\infty \Pr[\|Z\|\ge t] dt = n^{-\omega(1)}.
\end{align*}
and this completes the proof. 
\end{proof}

All of the random variables we consider are themselves products of subgaussian random variables, so they satisfy the tail bounds in the above lemma. In the remaining proofs we will focus on bounding the norm of these variables with high probability.
\subsubsection{ Proof of Lemma~\ref{lem:g_pertubation} and Lemma~\ref{lem:stepconcentration} }
%\begin{proof}[Proof of Lemma~\ref{lem:stepconcentration}] 
Since $s$ is fixed throughout, we will use $\tilA$ to denote $\tilA^s$. Also we fix $i$ in this proof. Let $S$ denote the support of $\trx$. Note that $\widehat{g}_i$ is a sum of random variable of the form $(y-\tilA \tilx)\sgn(\tilx_i)$. Therefore we are going to apply Bernstein inequality for proving $\widehat{g}_i$ concentrates around its mean $g_i$. Since Bernstein is typically not tight for sparse random varaibles like in our case. We study the concentration of the random variable $Z := (y-\tilA \tilx)\sgn(\tilx_i)\mid i\in S$ first. We prove the following technical lemma at the end of this section.  

\begin{claim}\label{lem:sample_conentration}
Let $Z^{(1)},\dots, Z^{(\ell)}$ be i.i.d random variables with the same distribution as $Z:= (y-\tilA \tilx)\sgn(\tilx_i)\mid i\in S$. Then when $\ell = \widetilde{\Omega}(k^2)$, 
$$\|\frac{1}{\ell}\sum_{j=1}^{\ell} Z^{(j)}- \Exp[Z] \|\le o(\delta_s)+ O(\sqrt{k/n})$$ 
\end{claim}

We begin by proving Lemma~\ref{lem:g_pertubation}. 

\begin{proof}[Proof of Lemma~\ref{lem:g_pertubation}]
Let $W = \{j: i \in \supp(\trxs{j})\}$ and then we have that $$\widehat{g}_i= \frac{|W|}{p}\cdot \frac{1}{|W|}\sum_j (y^{(j)}-\tilA \tilx^{(j)})\sgn(\tilx^{(j)}_i)$$ Note that $\frac{1}{|W|}\sum_j (y^{(j)}-\tilA \tilx^{(j)})\sgn(\tilx^{(j)}_i)$ has the same distribution as $\frac{1}{\ell}\sum_{j=1}^{\ell} Z^{(j)}$ for $\ell = |W|$, and indeed by concentration we have $\ell = |W| = \widetilde{\Omega}(k^2)$ when $p = \widetilde{\Omega}(mk)$. Also note that $\Exp[(y-\tilA \tilx)\sgn(\tilx_i)] = q_i\cdot \Exp[Z]$ with $q_i = O(k/m)$. Therefore by Lemma~\ref{lem:sample_conentration} we have that 
$$\|\widehat{g}_i - g_i\| \le O(k/m) \cdot\|\frac{1}{\ell}\sum_{j=1}^{\ell} Z^{(j)}- \Exp[Z] \| \le O(k/m)\cdot (o(\delta_s)+O(\sqrt{k/n}))$$ 
and this completes the proof. \end{proof}

%We can now combine the above claims and apply the vector Bernstein inequality.  Since $p = \widetilde{\Omega}(m)$ we have that with high probability $w_i = \widetilde{\Omega}(k)$. Then $\|g_i/w_i - \E[(y - \tilA^s \wtilx|i\in S]\| \le (o(\delta_s)+O(\sqrt{k/n}))$. 
\begin{proof}[Proof of Lemma~\ref{lem:stepconcentration}]
Therefore using Lemma~\ref{lem:OF-simplified-update} we can write $\widehat{g}_i^s$ (whp) as $\widehat{g}_i = \widehat{g}_i - g_i + g_i=  4\alpha(\tilA_i^s - \trA_i) + v$ with $\|v\| \le \alpha\|\tilA_i^s - \trA_i\|+O(k/m)\cdot (o(\delta_s)+O(\sqrt{k/n}))$. 
By Lemma~\ref{lem:checking_condtion1} we have $\widehat{g}_i$ %= \E[(y - \tilA\wtilx)\sgn(x_i) ]$ 
is $(\Omega(k/m), \Omega(m/k), o(k/m\cdot \delta_s^2) + O(k^2/mn))$-correlated-whp with $\trA_i$.  %and hence we have
\end{proof}
%\end{proof}
Then it suffices to prove Claim~\ref{lem:sample_conentration}. To this end, we apply the Bernstein's inequality stated in equation~\ref{eqn:bernstein} with the additional technical lemma~\ref{lem:general_bernstein}. We are going to control the maximum norm of $Z$ and as well as the variance of $Z$ using Claim~\ref{claim:stepinfinitynorm} and Claim~\ref{claim:var} as follows:

\begin{claim}
\label{claim:stepinfinitynorm}
 $\|Z\| = \|(y-\tilA\tilx)\sgn(\tilx_i)\| \le \widetilde{O}(\mu k/\sqrt{n} + k \delta_s^2 + \sqrt{k}\delta_s)$ holds with high probability
\end{claim}

\begin{proof}
We write $ y-\tilA\tilx = (\trA_S - \tilA_S \tilA_S^T \trA_S) \trx_S = (\trA_S - \tilA_S) \trx_S + \tilA_S(I - \tilA_S^T \trA_S) \trx_S$ and we will bound each term. For the first term, since $\tilA$ is $\delta_s$-close to $\trA$ and $|S|\le O(k)$, we have that $\|\trA_S-\tilA_S\|_F \le O(\delta_s \sqrt{k})$. And for the second term, we have
\begin{align*}
\|\tilA_S(\tilA_S^T \trA_S - I)\|_F
&\le \|\tilA_S\|\|(\tilA_S^T \trA_S - I)\|_F\\
& \le (\|\trA_S\|+\delta_s \sqrt{k}) (\| (\tilA_S - \trA_S)^T \trA_S\|_F + \|{\trA_S}^T \trA_S - I\|_F)\\
& \le (2+\delta_s \sqrt{k}) (\|\trA_S\|\|\tilA_S - \trA_S\|_F + \mu k/\sqrt{n})  \le O(\mu k/\sqrt{n} + \delta_s^2 k + \sqrt{k}\delta_s).
\end{align*}
Here we have repeatedly used the bound $\|U V\|_F \le \|U\|\|V\|_F$ and the fact that $\trA$ is $\mu$ incoherent which implies $\|\trA_S\| \le 2$. Recall that the entries in $\trx_S$ are chosen independently of $S$ and are subgaussian. Hence if $M$ is fixed then $\|M\trx_S\| \le \widetilde{O}(\|M\|_F)$ holds with high probability. And so
$$\|(y-\tilA\tilx)\sgn(\tilx_i)\| \leq \widetilde{O}(\|\trA_S-\tilA_S\|_F + \|\tilA_S(\tilA_S^T \trA_S - I)\|_F) \leq  \widetilde{O}(\mu k/\sqrt{n} + k \delta_s^2 + \sqrt{k}\delta_s)$$ which holds with high probability and this completes the proof. 
\end{proof}

Next we bound the variance.
\begin{claim}\label{claim:var}
$\Exp[\|Z\|^2] = \E[\|(y-\tilA\tilx)\sgn(\tilx_i)\|^2|i\in S] \le O(k^2\delta_s^2)  + O(k^3/n)$
%O(\delta_s^2 + k^2/n^2)$.
\end{claim}

\begin{proof}
We can again use the fact that $ y-\tilA\tilx = (\trA_S - \tilA_S \tilA_S^T \trA_S) \trx_S $ and that $\trx_S$ is conditionally independent of $S$ with $\Exp[\trx_S(\trx_S)^T]  = I$ and conclude

$$\E[\|(y-\tilA\tilx)\sgn(\tilx_i)\|^2|i\in S]  = \Exp[\|\trA_S - \tilA_S\tilA_S^T\trA_S\|_F^2\mid i\in S]$$
%\mbox{tr}\Big(\E_{S\ni i}[(\trA_S - \tilA_S \tilA_S^T \trA_S)(\trA_S - \tilA_S \tilA_S^T \trA_S)^T\mid i\in S]\Big)$$
%where $\mbox{tr}(\cdot)$ is the trace of a matrix and we used the fact that both the expectation and the trace are linear operators.

Then again we write $\trA_S - \tilA_S\tilA_S^T\trA_S$ as $(\trA_S - \tilA_S) + \tilA_S(I_{k\times k} - \tilA_S^T \trA_S)$, and the bound the Frobenius norm of the two terms separately. 
First, since $\tilA$ is $\delta_s$-close to $\trA$, we have that $\trA_S-\tilA_S$ has column-wise norm at most $\delta_s$ and therefore $\|\trA_S-\tilA_S\|_F \le \sqrt{k} \delta_s$. Second, note that $\|\tilA_S\|_F \le O(\sqrt{k})$ since each column of $\tilA$ has norm $1\pm \delta_s$, we have that
\begin{eqnarray*}
\Exp[\|\tilA_S(I - \tilA_S^T \trA_S)\|^2_F\mid i\in S] &\le & 
%\Exp[2\|\tilA_S\|^2 - \tilA_S^T \trA_S)\|^2_F] 
 O(k)\Exp[\|(I_{k\times k} - \tilA_S^T \trA_S)\|^2_F\mid i\in S] \\
 &\le & O(k) \Exp\left[\sum_{j\in S}(1-\tilA_j^T\trA_j)^2 + \sum_{j\neq \ell \in S} \langle\tilA_j,\trA_{\ell}\rangle^2\mid i\in S\right] \\
\end{eqnarray*}

We can now use the fact that $\tilA$ is $\delta_s$-close to $\trA$, expand out the expectation, and use the fact that $\Pr[j\in S, \ell\in S\mid i\in S]\le O(k^2/m^2)$, to obtain
\begin{eqnarray*}
&&\Exp[\|\tilA_S(I - \tilA_S^T \trA_S)\|^2_F\mid i\in S] \\
& \le& O(k^2\delta_s^2) + O(k^3/m^2)\cdot \sum_{j,\ell \in [m]\backslash i} \langle \tilA_j,\trA_{\ell} \rangle^2+ O(k^2/m) \|\tilA_i^T\trA_{-i}\|^2 + O(k^2/m)\|\tilA_{-i}^T\trA_i\|^2\\
&\le & O(k^2\delta_s^2)  + O(k^3/n)
\end{eqnarray*}
and this completes the proof. 
\end{proof}

\begin{proof}[Proof of Claim~\ref{lem:sample_conentration}]
We apply first Bernstein's inequality (\ref{eqn:bernstein}) with $R = \widetilde{O}(\mu k/\sqrt{n} + k \delta_s^2 + \sqrt{k}\delta_s)$ and $\sigma^2 = O(k^2\delta_s^2)  + O(k^3/n)$ on random variable $Z^{(j)}(1-\mathbf{1}_{\|Z^{(j)}\|\ge \Omega(R)})$. Then by claim~\ref{claim:stepinfinitynorm}, claim~\ref{claim:var} and Bernstein Inequality, we know that the truncated version of $Z$ concentrates when $\ell = \Omega(k^2)$, 
$$\left\|\frac{1}{\ell}\sum_{j=1}^{\ell} Z^{(j)}(1-\mathbf{1}_{\|Z^{(j)}\|\ge \Omega(R)})- \Exp[Z(1-\mathbf{1}_{\|Z\|\ge \Omega(R)})] \right\|\le \widetilde{O}\left(\frac{R}{\ell}\right) + \widetilde{O}\left(\sqrt{\frac{\sigma^2}{\ell}}\right) = o(\delta_s) + O(\sqrt{k/n}) $$
Note that we choose $\ell = k^2 \log^c n$ for a large constant $c$ so that it kills the log factors caused by Bernstein's inequality. 
%o(\delta_s) + O(k/n)$$ 
Then by Lemma~\ref{lem:general_bernstein}, we have that $\sum_j Z^{(j)}$ also concentrates: 
$$\|\frac{1}{\ell}\sum_{j=1}^{\ell} Z^{(j)}- \Exp[Z] \|\le o(\delta_s) + O(\sqrt{k/n}) $$ 
and this completes the proof. \end{proof}
\subsubsection{Proof of Lemma~\ref{lem:nearnessconcentration}}
\begin{proof}[Proof of Lemma~\ref{lem:nearnessconcentration}]
We will apply the matrix Bernstein inequality. In order to do this, we need to establish bounds on the spectral norm and on the variance. For the spectral norm bound, we have $\|(y - \tilA^s \wtilx) \sgn(\wtilx)^T\| = \|(y - \tilA^s \wtilx)\|\|\sgn(\wtilx)\| = \sqrt{k}\|(y - \tilA^s \wtilx)\|$. We can now use Claim~\ref{claim:stepinfinitynorm} to conclude that $\|(y - \tilA^s \wtilx)\| \le \widetilde{O}(k)$, and hence $\|(y - \tilA^s \wtilx) \sgn(\wtilx)\| \le \widetilde{O}(k^{3/2})$ holds with high probability.  

For the variance, we need to bound both $\E[(y - \tilA^s \wtilx) \sgn(\wtilx)^T\sgn(\wtilx)(y - \tilA^s \wtilx)^T]$ and $\E[\sgn(\wtilx)(y - \tilA^s \wtilx)^T(y - \tilA^s \wtilx) \sgn(\wtilx)^T]$. The first term is equal to $k\E[(y - \tilA^s \wtilx)(y - \tilA^s \wtilx)^T]$. Again, the bound follows from the calculation in Lemma~\ref{lem:OFupdate} and we conclude that $$\|\E[(y - \tilA^s \wtilx) \sgn(\wtilx)^T\sgn(\wtilx)(y - \tilA^s \wtilx)^T] \|\leq O(k^2/n)$$
To bound the second term we note that $$\E[\sgn(\wtilx)(y - \tilA^s \wtilx)^T(y - \tilA^s \wtilx) \sgn(\wtilx)^T] \preceq \widetilde{O}(k^2)\E[\sgn(\wtilx)\sgn(\wtilx)^T] \preceq \widetilde{O}(k^3/m) I$$ Moreover we can now apply the matrix Bernstein inequality and conclude that when the number of samples is at least $\widetilde{\Omega}(mk)$ we have $$\|\frac{1}{p} \sum_{j=1}^p (y^{(j)} - \tilA^s \trxs{j}) \sgn(\trxs{j})^T - \E[(y - \tilA^s \wtilx) \sgn(\wtilx)^T\| \le O^*(k/m \cdot \sqrt{m/n})$$ and this completes the proof. 
\end{proof}
\subsubsection{Proof of Lemma~\ref{lem:initconcentration}}\label{sec:initconcentration}

%\begin{proof}[Proof of Lemma~\ref{lem:initconcentration}]
Again in order to apply the matrix Bernstein inequality we need to bound the spectral norm and the variance of each term of the form $\langle u,y\rangle\langle v,y\rangle yy^T$ . We make use of the following claim to bound the magnitude of the inner product:

\begin{claim}
$|\langle u,y\rangle| \le \widetilde{O}(\sqrt{k})$ and $\|y\| \le \widetilde{O}(\sqrt{k})$ hold with high probability
\end{claim}

\begin{proof}
Since $u = \trA\alpha$ and because $\alpha$ is $k$-sparse and has subgaussian non-zero entries we have that $\|u\| \le \widetilde{O}(\sqrt{k})$, and the same bound holds for $y$ too. Next we write $|\langle u,y\rangle| = |\langle {\trA}_S^Tu, \trx_S\rangle|$ where $S$ is the support of $\trx$. Moreover for any set $S$, we have that $$\|{\trA_S}^Tu\| \le \|\trA_S\|\|u\| \le \widetilde{O}(\sqrt{k})$$ holds with high probability, again because the entries of $\trx_S$ are subgaussian we conclude that $|\langle u,y\rangle|\le \widetilde{O}(\sqrt{k})$ with high probability.
\end{proof}

\noindent This implies that $\|\langle u,y\rangle\langle v,y\rangle yy^T \| \leq \widetilde{O}(k^2)$ with high probability.

Now we need to bound the variance:

\begin{claim}
$\|\E[\langle u,y\rangle^2\langle v,y\rangle^2 yy^Tyy^T]\| \le \widetilde{O}(k^3/m)$
\end{claim}

\begin{proof}
We have that with high probability $\|y\|^2 \le \widetilde{O}(k)$ and $\langle u,y\rangle^2 \le \widetilde{O}(k)$, and we can apply these bounds to obtain $$\E[\langle u,y\rangle^2\langle v,y\rangle^2 yy^Tyy^T] \preceq \widetilde{O}(k^2)\E[\langle v,y\rangle^2 yy^T]$$ On the other hand, notice that $\E[\langle v,y\rangle^2 yy^T] = M_{v,v}$ and using Lemma~\ref{lem:initform} we have that $\|\E[\langle v,y\rangle^2 yy^T]\| \le O(k/m)$. Hence we conclude that the variance term is bounded by $\widetilde{O}(k^3/m)$.
\end{proof}

Now we can apply the matrix Bernstein inequality and conclude that when the number of samples is $p = \widetilde{\Omega}(mk)$ then $$\|\widehat{M}_{u,v} - M_{u,v}\| \leq \widetilde{O}(k^2)/p + \sqrt{\widetilde{O}(k^3/mp)} \le O^*(k/m\log n)$$ with high probability, and this completes the proof. 
%\end{proof}

%% file: neuralarchitecture.tex
\section{Neural Implementation}\label{sec:neural}

	\begin{figure}
   				\begin{center}\label{fig:neural_architecture}
               \includegraphics[width=1\textwidth]{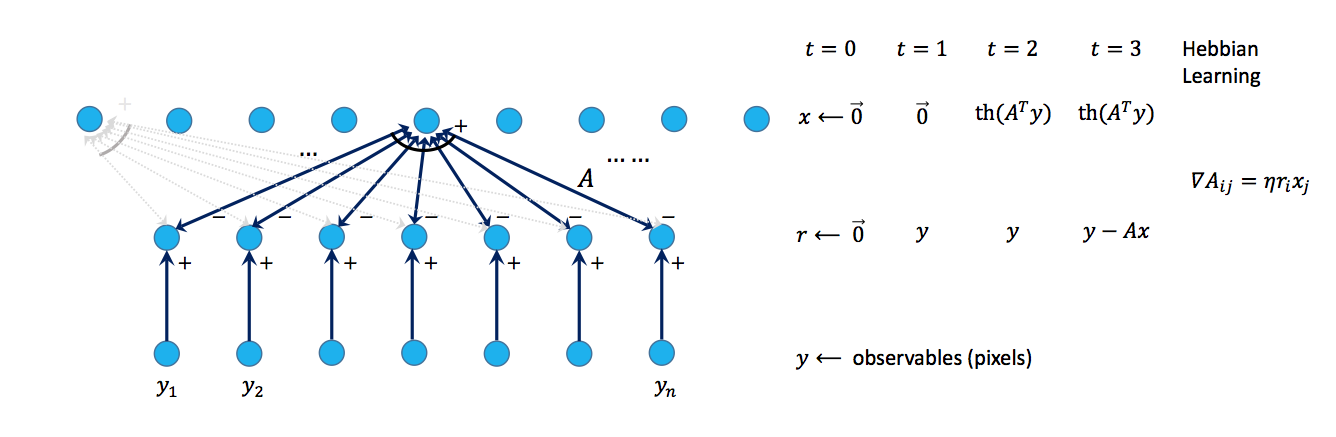} 
                \end{center}
      \caption{A neural implementation of Algorithm~\ref{eqn:simplestupdate}, which mimics that of Olshausen-Field  (Figure 5 in \cite{OF})}
       \end{figure}
       
           \paragraph{Neural Implementation of Alternating Minimization:}  Here we sketch a neural architecture implementing Algorithm~\ref{eqn:simplestupdate}, essentially mimicking Olshausen-Field  (Figure 5 in \cite{OF}), except our decoding rule is much simpler and takes
       a single neuronal step. 
      % \vspace{0.5pc}
      
             \vspace{0.2pc}
       
\textbf{(a)} The bottom layer of neurons take input $y$ and at the top layer neurons output the decoding $x$ of $y$ with respect to the current code. The middle layer labeled $r$ is used for intermediate computation.  The code $A$ is stored as the weights between the top and middle layer on the synapses. Moreover these weights are set via a Hebbian rule, and upon receiving a new sample $y$, we update the weight $A_{ij}$ on the synapses by the product of the values of the two endpoint neurons, $x_j$ and $r_i$.  

       \vspace{0.2pc}
      
       \textbf{(b)} The top layer neurons are equipped with a threshold function. 
            The middle layer ones are equipped with simple  linear functions (no threshold). 
            The bottom layer can only be changed by stimulation from outside the system.
            
                   \vspace{0.2pc}
            
    \textbf{(c)} We remark that the updates require some attention to timing, which can be accomplished via spike timing.
      %more detail, our network requires a clock for synchronization. 
     In particular, when a new image is presented, the value of all neurons are updated to a (nonlinear) function of the weighted sum of the values of its neighbors with weights on the corresponding synapses. 
      The execution of the network is shown at the right hand side of the figure. Upon receiving a new sample at time $t = 0$, the values of bottom layer are set to be $y$ and all the other layers are reset to zero. At time $t=1$, the values in the middle layer are updated by the weighted sum of their neighbors, which is just $y$. Then at time $t =2$, the top layer obtains the decoding $x$ of $y$ by calculating $\textrm{threshold}_{C/2}(A^Ty)$. At time $t=3$ the middle layer   calculates the residual error $y - Ax$ and then at time $t=4$ the  synapse weights that
store $A$ are updated via Hebbian rule (update proportional to the product of the endpoint values). Repeating this process with many images indeed implements Algorithm~\ref{eqn:simplestupdate} and succeeds provided that the network is appropriately initialized to $\tilA^0$ that is $(\delta, 2)$ near to $\trA$. 
      
      %The architecture for computing the intial $\tilA^0$ is omitted here and uses the usual neural algorithms for computing top singular vectors.
      
      \paragraph{Neural Implementation of Initialization:} Here we sketch a neural implementation of Algorithm~\ref{thm:initmain}. This algorithm uses simple operations that can be implemented in neurons and composed, however unlike the above implementation we do not know of a two layer network, but note that this procedure need only be performed once so it need not have a particularly fast or short neural implementation. 
      
             \vspace{0.2pc}
      
       \textbf{(a)} It is standard to compute the inner products $\inner{y, u}$ and $\inner{y, v}$ using neurons, and even the top singular vector can be computed using  the classic Oja's Rule~\cite{oja82} in an online manner where each sample $y$ is received sequentially. There are also generalizations to computing other principle components~\cite{oja92}. However, we only need the top singular vector and the top two singular values. 
      
             \vspace{0.2pc}
      
       \textbf{(b)} Also, the greedy clustering algorithm which preserves a single estimate $z$ in each equivalence class of vectors that are $O^*(1/\log m)$-close (after sign flips) can be implemented using inner products. Finally, projecting the estimate $\widetilde{A}$ onto the set $\mathcal{B}$ may not be required in real life (or even for correctness proof), but even if it is it can be accomplished via stochastic gradient descent where the gradient again makes use of the top singular vector of the matrix.